\documentclass{article}

\usepackage{microtype}
\usepackage{graphicx}
\usepackage{subfigure}
\usepackage{booktabs} 
\usepackage{tcolorbox}
\usepackage{amsmath}
\usepackage{amsthm}
\usepackage{bbm}
\usepackage{bm}
\usepackage{multirow}
\usepackage{xfrac}
\usepackage{amssymb}  
\usepackage{tabularx} 
\usepackage{amsmath}  
\usepackage{mathtools}

\newtheorem{theorem}{Theorem}

\newtheorem{corollary}{Corollary}
\newtheorem{lemma}{Lemma}

\newtheorem{definition}{Definition}
\newtheorem{remark}{Remark}

\DeclareMathOperator*{\argmax}{arg\,max}
\DeclareMathOperator*{\argmin}{arg\,min}
\usepackage{mathtools}
\allowdisplaybreaks
\DeclarePairedDelimiter{\ceil}{\lceil}{\rceil}
\DeclarePairedDelimiter{\floor}{\lfloor}{\rfloor}
\usepackage{hyperref}

\usepackage[accepted]{icml2020}

\icmltitlerunning{Kernel Methods for Cooperative Multi-Agent Contextual Bandits}

\begin{document}

\twocolumn[
\icmltitle{Kernel Methods for Cooperative Multi-Agent Contextual Bandits}

\begin{icmlauthorlist}
\icmlauthor{Abhimanyu Dubey}{to}
\icmlauthor{Alex Pentland}{to}
\end{icmlauthorlist}

\icmlaffiliation{to}{Media Lab and Institute for Data, Systems and Society,
  Massachusetts Institute of Technology}



\icmlkeywords{multi-agent learning, online learning, contextual bandits}

\vskip 0.3in
]


\icmlcorrespondingauthor{Abhimanyu Dubey}{dubeya@mit.edu}
\printAffiliationsAndNotice{}  

\begin{abstract}
Cooperative multi-agent decision making involves a group of agents cooperatively solving learning problems while communicating over a network with delays. In this paper, we consider the kernelised contextual bandit problem, where the reward obtained by an agent is an arbitrary linear function of the contexts' images in the related reproducing kernel Hilbert space (RKHS), and a group of agents must cooperate to collectively solve their unique decision problems. For this problem, we propose \textsc{Coop-KernelUCB}, an algorithm that provides near-optimal bounds on the per-agent regret, and is both computationally and communicatively efficient. For special cases of the cooperative problem, we also provide variants of \textsc{Coop-KernelUCB} that provides optimal per-agent regret. In addition, our algorithm generalizes several existing results in the multi-agent bandit setting. Finally, on a series of both synthetic and real-world multi-agent network benchmarks, we demonstrate that our algorithm significantly outperforms existing benchmarks.
\end{abstract}

\section{Introduction}
An emerging problem in online learning and multi-agent distributed systems is the \textit{cooperative} multi-agent bandit. It involves a group $\mathcal V$ of $V$ agents collectively solving a decision problem while communicating with each other. The problem proceeds in rounds $t = 1, 2, ..., T$, where at any trial $t = 1, 2, ...$, each agent $v \in \mathcal V$ is presented with a \textit{decision set} $\mathcal D_{v, t}$, and selects an action $\bm x_{v, t} \in \mathcal D_{v, t}$. Each agent obtains a stochastic reward $y_{v, t}$, following:
\begin{equation*}
    y_{v, t} = f(\bm x_{v, t}) + \varepsilon_{v, t},
\end{equation*} where $\varepsilon_{v,t}$ is i.i.d. noise, and $f$ is an unknown (but fixed) function. The collective objective of the group of agents is to select actions that minimize the expected \textit{group regret}:
\begin{equation*}
    \mathcal{R}_{\mathcal G}(T) = \sum_{v \in \mathcal V}\sum_{t=1}^T \left(f(\bm x^*_{v, t}) - f(\bm x_{v, t})\right),
\end{equation*}
where, $\bm x^*_t = \argmax_{\bm x \in D_{v, t}} f(\bm x)$. The research objective for this problem is to design multi-agent algorithms that can leverage communication to improve overall performance~\cite{landgren2016distributed, pmlr-v28-szorenyi13}.

Agents communicate via an undirected graph $\mathcal G = (\mathcal V, \mathcal E)$, where $(i, j) \in \mathcal E$ if agents $i$ and $j$ are connected. Messages from any agent $v$ are available to agent $v'$ after $d(v, v') - 1$ trials of the bandit, where $d$ is the distance between the agents in $\mathcal G$. This gradually creates \textit{heterogeneity} between the information available to each agent, and is the primary technical challenge in algorithm design for this problem. Moreover, recent work assumes that the bandit problem is common (i.e., $f$ is identical for all agents), but this assumption does not hold for most decentralized applications~\cite{boldrini2018mumab}. For instance, in a decentralized supply chain network~\cite{thadakamaila2004survivability}, agents interact with similar but non-identical decision problems, since loads are generally distributed non-uniformly. In this setting, na\"ively incorporating observations from neighboring agents may not be beneficial, and algorithms must be carefully designed to optimally leverage cooperation.

A related problem is the online \textit{social network clustering} of bandits, where, at every trial, a randomly selected agent interacts with the bandit~\cite{cesa2013gang, gentile2014online, gentile2017context, li2016collaborative, li2019improved}. In this formulation, a fixed (but unknown) clustering over the agents is assumed, where agents within a cluster have identical context functions. While the assumptions of linearity and clustering are feasible in the context of social networks~\cite{al2006clustknn}, these assumptions may not hold for general multi-agent environments, such as geographically-distributed computational clusters~\cite{cano2016towards}. In the case when each agent has its own unique decision problem, the clustering approach leads to an $\mathcal O(V)$ multiplicative increase in the group regret. Moreover, the \textit{social network clustering} problem is \textit{single-agent}, since at any trial, only one agent interacts with the bandit. This makes it less challenging compared to the multi-agent cooperative bandit, since there is no heterogeneity of information (as discussed earlier). Multi-agent settings have been considered for \textit{social network clustering}~\cite{korda2016distributed}, but without delayed feedback (hence, without heterogeneity).

\textbf{Contributions}. In this paper, we study the cooperative multi-agent bandit with delays. We assume that each agent $v \in \mathcal V$ interacts with a separate bandit function $f_v$, where all functions $f_v, \ v \in \mathcal V$ have small norm in a known reproducing kernel Hilbert space (RKHS)~\cite{scholkopf2005support} specified by a fixed kernel $K_x$. This is a more general setting compared to the existing \textit{clustering} or identical (i.e., \textit{fully-cooperative}) settings in the literature, and allows us to propose a technique to measure the similarity between the functions $f_v$ via an agent-based similarity kernel, which can be learnt online when it is unknown. Under this formulation, we present \textsc{Coop-KernelUCB}, an algorithm for the multi-agent contextual bandit problem on networks.

For the context-free multi-agent cooperative bandit problem, existing bounds on group regret scale as $\mathcal O\left(\sqrt{TV\log(V\cdot\lambda_{\max}(\mathcal G))}\right)$, where, $\lambda_{\max}(\mathcal G)$ is the maximum eigenvalue of the graph Laplacian for $\mathcal G$~\cite{martinez2018decentralized, landgren2016distributed2, landgren2016distributed}. These bounds are obtained by using a communication protocol known as the \textit{running consensus}, that involves agents averaging their beliefs with neighboring agents. This communication protocol restricts these methods to the \textit{fully-cooperative} setting, i.e., all agents have identical arms; it is trivial to see that naive averaging of estimates in non-identical settings can lead to irreducible bias and $\Theta(T)$ regret. \textsc{Coop-KernelUCB} employs an alternate communication protocol, entitled~\textsc{Local} \cite{suomela2013survey}, that involves agents sending messages to each other. Additionally, \textsc{Coop-KernelUCB} uses an alternative ``network'' kernel $K_z$, to measure similarity between agent reward functions $f_1, ..., f_V$. When $K_z$ is known (such as, e.g., in cases when agents correspond to users in a social network), we can use $K_z$ to construct a product kernel $K = K_z \odot K_x$, and use $K$ (instead of $K_x$) to construct upper confidence bounds. 

We consider the case when the decision sets $\mathcal D_{v, t}$ can be infinite or continuum action spaces, and $\forall \ v \in \mathcal V, \lVert f_v \rVert_{\mathcal H} \leq B$. In this setting, the single-agent \textsc{IGP-UCB} \cite{chowdhury2017kernelized} algorithm, assuming the agent interacts with a single function $f$, obtains a regret of $\widetilde{\mathcal O}(\sqrt{VT}(B\sqrt{\Upsilon^x_{VT}} + \Upsilon^x_{VT}))$ (where the $\widetilde{\mathcal O}$ hides additional logarithmic dependence on $1/\delta$) when run for a total of $VT$ rounds (i.e., a centralized agent that pulls all arms), where $\Upsilon^x_{VT}$ is the \textit{information gain} after $VT$ rounds, a quantity dependent on the structure of the RKHS of $\mathcal X$. We demonstrate that \textsc{Coop-KernelUCB}, that uses the underlying martingale inequality from~\textsc{IGP-UCB}, obtains a regret of $\widetilde{\mathcal  O}\left(\sqrt{VT\cdot\bar\chi(\mathcal G_\gamma)}\left(B\sqrt{\Upsilon^x_{VT} \Upsilon^z} + \Upsilon^x_{VT} \Upsilon^z\right)\right)$. Here, $\Upsilon^z$ is a term corresponding to the similarity between functions $f_v$ via the kernel $K_z$, and $\bar\chi(\mathcal G_\gamma)$ is a term accounting for the delayed propagation of information in the network $\mathcal G$\footnote{$\bar\chi(\mathcal G_\gamma)$ denotes the minimum clique number of the $\gamma^{th}$ power of graph $\mathcal G$, i.e., $\mathcal G_\gamma$ has an edge $(i, j)$ if there is a path of length at most $\gamma$ between $i$ and $j$ in $\mathcal G$.}. This bound is achieved by utilizing graph partitions to control the deviation in the confidence bound for each agent.

Our bound is reminiscent of single-agent bounds with additional contexts~\cite{deshmukh2017multi, krause2011contextual}, which rely on a \textit{known} $K_z$. However, in many cases, $K_z$ is (unknown) and requires estimation. For this case, we provide an alternative algorithm via kernel mean embeddings~\cite{christmann2010universal}. 
Against state-of-the-art methods on a variety of real-world and synthetic multi-agent networks, our algorithm exhibits superior performance. Moreover, we present a variant, \textsc{Eager-KernelUCB}, of our algorithm (without regret bounds) that comfortably outperforms \textsc{Coop-KernelUCB} and other benchmarks. This extends the current literature of cooperative bandit estimation from the stochastic multi-armed problem~\cite{landgren2018social, martinez2018decentralized, landgren2016distributed} to a more general class of functions, and provides a technique to determine task similarity over arbitrary cooperative settings.

\section{Preliminaries}

\textbf{Notation}. We use boldface uppercase to represent matrices, i.e., $\bm A$, and lowercase for vectors, i.e., $\bm x$. The RKHS norm of a function $f$ in RKHS $\mathcal H$ with kernel $K$ is given by $\lVert f \rVert_{\mathcal H} = \sqrt{\langle f, f \rangle_{\mathcal H}}$. We denote the set $\{a, a+1 ..., b-1, b\}$ by the shorthand $[a, b]$ and simply as $[b]$ when $a=1$. We refer to the $\gamma^{th}$ graph power of a graph $\mathcal G$ as $\mathcal G_\gamma$ (i.e., $\mathcal G_\gamma$ contains an edge $(i, j)$ if there exists a shortest path of length at most $\gamma$ between $i$ and $j$ in $\mathcal G$). We denote a clique in $\mathcal G_\gamma$ as a $\gamma$-clique in $\mathcal G$, and denote $\mathcal G$ as $\gamma$-complete if $\mathcal G_\gamma$ is complete. $N_{\gamma}(v) \subseteq \mathcal V$ denotes the set of nodes at a shortest distance of at most $\gamma$ from node $v$ in $\mathcal G$, referred to as the ``$\gamma$-neighborhood'' of $v$ in $\mathcal G$. The distance between two nodes $v$ and $v'$ in $\mathcal G$ is given by the shorthand $d_{\mathcal G}(v, v')$.

\textbf{Problem Setup}. We consider a multi-agent setting of $V$ agents sitting on the vertices of a network represented by an undirected and connected graph $\mathcal G = (\mathcal V, \mathcal E)$\footnote{Our results and algorithm can be trivially extended to directed or disconnected case, by considering each connected subgraph individually, and considering statistics of the directed graph instead.}. We assume that agents each solve unique instances of kernelised contextual bandit problems.  At each step $t = 1, 2, ...$ each agent $v \in V$ obtains, at time $t$, a decision set $\mathcal D_{v, t} \subseteq \mathcal X \subset \mathbb R^d$, where $\mathcal X$ is a compact subset of $\mathbb R^d$. In this paper, we assume $\mathcal D_{v, t}$ to even be an infinite set or a continuum of actions, however, we discuss the case when it is a finite set of contexts $\bm x^{(1)}_{v, t}, \bm x^{(2)}_{v, t}, ... $ later on, for which tighter regret bounds can be obtained. At each trial $t$, each agent selects an action $\bm x_{v, t} \in \mathcal D_{v, t}$, and receives a reward $y_{v, t} = f_v\left({\bm x}_{v, t}\right) + \varepsilon_{v, t}$. Where $f_v : \widetilde{\mathcal X} \rightarrow \mathbb R$ is a fixed (but unknown) function, and $\varepsilon_{v, t}$ is additive noise such that the noise sequence $\left\{ \varepsilon_{v, t}\right\}_{t=1, v \in V}^\infty$ is conditionally $R$-sub-Gaussian.

\textbf{Single-Agent UCB}.
Our approach builds on the existing research for upper confidence bounds for bandit kernel learning, a line of research that has seen a lot of interest~\cite{srinivas2009gaussian, krause2011contextual, chowdhury2017kernelized, valko2013finite, deshmukh2017multi}. The central idea across all these approaches is to construct an upper confidence bound (UCB) envelope for the true function $f(\cdot)$ using an estimate $\hat{f}_t$, and then chooses an action $\bm x_t \in \mathcal D_t$ that maximizes this upper confidence bound, i.e., for some estimate $\hat{f}_t$ of $f$,
\begin{equation}
    \bm x_t = \argmax_{\bm x \in \mathcal D_{t}} \left[\hat f_{t}\left(\bm x\right) + \sqrt{\beta_{t}}\sigma_{t-1}\left(\bm x\right)\right].
\end{equation}
Here, $\beta_t$ is an appropriately chosen ``exploration'' parameter, and $\sigma_{t-1}$ can be thought of as the ``variance'' in the estimate $\hat f_t$. Existing UCB-based approaches aim to construct a sequence $(\beta_t)_t$ to ensure a near-optimal tradeoff between exploration and exploitation. The natural choice for $\hat f_t$ is the solution to the kernelised ridge regression. $\text{Given }\lambda \geq 0\text{ and }{\bm X}_{<t} = ({\bm x}_i, y_i)_{i = 1}^{t},$
\begin{equation}
\label{eqn:ridge_regression}
\hat f_{t} = \argmin_{f \in \mathcal H} \frac{1}{t}\sum_{({\bm x}, y) \in {\bm X}_{<t}} \left(f({\bm x}) - y\right))^2 + \lambda \lVert f \rVert_{\mathcal H}^2.
\end{equation}
The solution to the above problem (\ref{eqn:ridge_regression}) can be written as the following~\cite{valko2013finite} ($\text{for } \bm \kappa_{t}(\bm{x}) = \left(K\left(\bm x, {\bm x}_i\right)\right)_{i=1}^{t}, \bm y_{t} = \left(y_{i}\right)_{i=1}^{t} \text{ and } \bm K_t = (K(\bm x_i, \bm x_j))_{i, j \in [t]})$:
\begin{equation}
    \hat f_{t}(\bm{x}) = \bm \kappa_{t}(\bm{x})^\top \left(\bm K_{t} + \lambda \bm I\right)^{-1}\bm y_{t}.
\end{equation}
For any particular choice of the sequence $(\beta_t)_t$, various algorithms can be obtained, with different regret guarantees. To provide more insight into the regret bounds obtained by various algorithms, we now provide the definition of $\Upsilon^x$, i.e., \textit{maximum information gain}.
\begin{definition}[Maximum Information Gain~\cite{srinivas2009gaussian, krause2011contextual}]
For $y_t = f(\bm x_t) + \varepsilon_t$, let $A \subset \mathcal X$ be a finite subset such that $|A| = T$. Let $\bm y_A = \bm f_A + \bm\varepsilon_A$ where $\bm f_A = (f(\bm x_i))_{\bm x_i \in A}$ and $\bm\varepsilon_A \sim \mathcal N(0, R^2)$. The maximum information gain $\Upsilon_T^x$ after $T$ rounds is:
\begin{equation}
    \Upsilon_T^x \triangleq \max_{A \subset \mathcal X : |A|=T} H(\bm y_A) - H(\bm y_A | f).
\end{equation}
Here $H(\cdot)$ refers to the entropy of a random variable. Furthermore, if we assume that the kernel $K_x$ is bounded; i.e., $K_x(\bm x, \bm x) \leq 1 \ \forall \bm x \in \mathcal X$, the following is true for compact and convex $\mathcal X \subset \mathbb R^d$. For linear $K_x$, $\Upsilon^x_T = \mathcal O(d\log T)$. For RBF $K_x$, $\Upsilon^x_T = \mathcal O((\log T)^{d+1})$. For Mat\'ern $K_x$ with $\nu > 1$, $\Upsilon^x_T = \mathcal O(T^{\frac{d(d+1)}{2\nu + d(d+1)}}(\log T))$.
\end{definition}
\begin{remark}[UCB Regret for Single-Agent Algorithms]
\label{remark:single_agent}
Let $\delta \in (0, 1]$. For continuum-armed $\mathcal D_t$, choosing $\beta_t = 2B + 300\Upsilon^x_{t-1}log^3(t/\delta)$ guarantees with probability at least $1-\delta$ a regret of $\widetilde{\mathcal O}(\sqrt{T}(B\sqrt{\Upsilon^x_T} + \Upsilon^x_T\ln^{3/2}(T)))$, as demonstrated in the work of~\cite{srinivas2009gaussian} (GP-UCB). This was improved via a new martingale inequality to $\widetilde{\mathcal O}(\sqrt{T}(B\sqrt{\Upsilon^x_T} + \Upsilon^x_T))$ in the work of~\citet{chowdhury2017kernelized} with the choice of $\beta_t = B+ R\sqrt{2(\Upsilon^x_{t-1} + 1 + \ln(1/\delta))}$. An alternative regret bound of $\widetilde{\mathcal O}(\sqrt{\tilde{d}T})$ was provided via the Sup-KernelUCB algorithm of~\citet{valko2013finite} (for finite-armed $\mathcal D_t$), where $\tilde{d}$ is the \textbf{effective dimension} of $K_x$, a measure of the intrinsic dimensionality of the RKHS $\mathcal H$. $\tilde{d}$ is related to $\Upsilon^x$ as $\Upsilon^x \geq \Omega(\tilde{d}\ln\ln T)$. Further work has focused on improving bounds for various families of kernels, e.g., see~\cite{janz2020bandit, scarlett2017lower}.
\end{remark}
The primary goal in the cooperative learning setting is to provide each agent with stronger estimators that leverage observations from neighboring agents. A suitable baseline, therefore, in this setting, would be that of a centralized agent pulling $VT$ arms in a round-robin manner. Existing single-agent algorithms propose a regret bound of $\widetilde{\mathcal O}(\Upsilon^x_{VT}\sqrt{VT})$ in this setting, and this is the comparative regret bound we wish to match. We do not focus on stronger controls for specific kernels or of the information gain $\Upsilon^x$, and for that we refer the reader to references in Remark~\ref{remark:single_agent}.

\section{Cooperative Kernelized Bandits}
\textbf{Network Contexts.} Recall that for any agent $v$, the rewards $y_v$ are generated following $y_{v, t} = f_v(\bm x_{v, t}) + \varepsilon_{v, t}$. To provide a relationship between different $f_v$, we assume that the functions $f_v, v \in \mathcal V$ are parameteric functionals of some function $F: \mathcal X \times \mathcal Z \rightarrow \mathbb R$ for a known \textit{network context} space $\mathcal Z$ such that $\forall \ v \in \mathcal V, \exists \ \bm z_v \in \mathcal Z$ such that $\forall \bm x \in \mathcal X$,
\begin{align}
    f_v(\bm x) = F(\bm x, \bm z_v).
\end{align}
\textbf{Kernel Assumptions}. We denote the space $\mathcal X \times \mathcal Z$ as $\widetilde{\mathcal X}$, and the overall input $(\bm x, \bm z)$ as $\tilde{\bm x}$. Furthermore, we assume that the function $F$ has a small norm in  the reproducing kernel Hilbert space (RKHS,~\citet{scholkopf2005support}) $\mathcal H_{K}$ associated with a PSD kernel $K : \widetilde{\mathcal X} \times \widetilde{\mathcal X} \rightarrow \mathbb R$. $\mathcal H_{K}$ is completely specified by the kernel $K(\cdot, \cdot)$, and via an inner product $\langle \cdot, \cdot \rangle_{K}$ following the reproducing property. As is typical with the kernelized bandit literature, we assume a known bound on the RKHS norm of $F$, i.e. $\lVert F \rVert_{K} \leq B$, and we assume that the kernel has finite variance, i.e. $K(\tilde{\bm x}, \tilde{\bm x}) \leq 1, \forall \ \tilde{\bm x} \in \widetilde{\mathcal X}$\footnote{These are typically made assumptions in the contextual bandit literature, and avoid scaling of the regret bounds. In the linear case, the first assumption corresponds to having a bound on the norm of the context vectors~\cite{chowdhury2017kernelized}, and the second is to ensure the methods are scale-free~\cite{agrawal2012analysis}.}.

Finally, we must impose constraints on the interaction of the inputs $\bm x$ and $\bm z$ via two kernels $K_x(\cdot, \cdot)$ and $K_z(\cdot, \cdot)$. We assume that $K$ is a composition of two separate positive-semidefinite kernels, $K_z$ and $K_x$ such that $K_z : \mathcal Z \times \mathcal Z \rightarrow \mathbb R$, i.e., operating on the network contexts, and $K_x : \mathcal X \times \mathcal X \rightarrow \mathbb R$ operates on the action contexts. Our regret bounds assume that the overall kernel $K$ is formed via the Hadamard product of $K_x$ and $K_z$:
\begin{equation}
    K\left((\bm z, \bm x), (\bm z', \bm x')\right) = K_x(\bm x, \bm x') K_z(\bm z, \bm z').
\end{equation}
\begin{remark}[Kernel Compositions]
For the development in the paper, we restrict ourselves to the Hadamard composition, however, it is important to note that this is not a limitation of our technique, and other compositions can be explored. See the Appendix for details on the sum ($K_z \oplus K_x$), and Kronecker ($K_z \otimes K_x$) compositions.
\end{remark}
\begin{remark}[Independent vs. Pooled Modeling]
When $\mathcal D_{v, t}$ are countably finite, an alternate formulation is the ``independent'' assumption~\cite{li2010contextual}, where a separate model is considered for each ``arm''. We assume the ``pooled'' environment~\cite{abbasi2011improved}, (i.e., where all ``arms'' are modeled together), however it is easy to extend results to the former setting, by assuming arm-dependent network contexts (see appendix).
\end{remark}

The \textit{network} kernel, $K_z$, determines how ``similar'' agent functions $f_v$ are. For example, if all agents solve the same bandit problem, i.e., $f_v = f \ \forall v \in V$, then the appropriate choice for this is to set $\mathcal Z = \{1\}$, and $\bm z_{v} = 1$ for all $ v \in V$, and hence, $K = K_x$. Alternatively, in many internet applications, users (which may correspond to agents) are arranged in an online social network (say, $\mathcal G_{\text{net}}$), and $\bm z_v$ can be a network embedding of user $v$ in $\mathcal G_{\text{net}}$. Typically, however, $\mathcal Z$ can be defined more generally with a corresponding positive semi-definite (PSD) kernel $K_z$.  

\subsection{\textsc{Local} Communication}
Existing research on distributed bandit learning has largely focused on two communication protocols - the first being a centralized setting~\cite{liu2010distributed, wang2020distributed}, as is standard in distributed computation, where a central server acts as an intermediary between ``client'' agents (i.e., a star-graph communication), and the second being the \textit{running consensus} protocol, where agents are arranged in a network structure, but communication is done by repeatedly averaging (a weighted version) of observations with neighbors~\cite{landgren2016distributed, landgren2016distributed2}. 

In this paper, we use the \textsc{Local} communication protocol~\cite{suomela2013survey, fraigniaud2016locality, linial1992locality}, which has recently seen an increase in interest in the decentralized multi-agent bandit literature~\cite{cesa2019cooperative, cesa2019delay}. At an abstract level, it can be seen as a generalization of the centralized communication protocol to a server-free, arbitrary graph setting. The protocol assumes that pulling a bandit arm and communication occur sequentially within each trial $t$, i.e., first, each agent $v \in \mathcal V$ pulls an arm $\tilde{\bm x}_{v, t}$ and receives a reward $y_{v, t}$ from the respective bandit environment. The agent then sends the message $\bm m_{v, t} = \left\langle t, v, \tilde{\bm x}_{v, t}, y_{v, t}\right\rangle$ to its neighbors in $\mathcal G$. This message is forwarded from agent to agent $\gamma$ times (taking one trial of the bandit problem each between forwards), after which it is dropped. The \textit{time-to-live} (delay) parameter $\gamma$ is a common technique to control communication complexity in this setting. Each agent $v \in V$ therefore also receives messages $\bm m_{v', t-d(v, v')}$ from all the nodes $v'$ such that $d(v, v') \leq \gamma$.

\subsection{Cooperative Kernel-UCB}
In this section we present the primary algorithm, cooperative Kernel-UCB. The central ideas in the development of the algorithm are (a) to leverage the similarity of the agent kernels (as specified by $K_z$) and (b) to control the variance estimates $\sigma^2_{v, t-1}$ between agents by delayed diffusion of rewards. 

\textbf{Using an Augmented Kernel}. For each agent $v \in \mathcal V$ we construct an upper confidence bound (UCB) envelope for the true function $f_v(\cdot) = F(\cdot, \bm z_v)$ over the space $\widetilde{\mathcal X}$. This is done by using the composition kernel $K$ instead of the action kernel $K_x$, which allows us to take the network context $\bm z_v$ into account. The agent then chooses an action that maximizes the upper confidence bound, following the typical approach in UCB-based algorithms. For any $v \in V, \bm x \in \mathcal D_{v, t}$, the UCB can be given by,
\begin{equation}
{\bm x}_{v, t} = \argmax_{\bm x \in \mathcal D_{v, t}} \left[\widehat F_{v, t}\left(\bm z_v, \bm x\right) + \sqrt{\beta_{v, t}}\sigma_{v, t-1}\left(\bm z_v, \bm x\right)\right].
\end{equation}
Here $\widehat F_{v, t}(\cdot, \bm z_v)$ is the agent's estimate for $f_v$ at time $t$, and the second term denotes the exploration bonus. Using ${\bm x}_{v, t}$, the agent can construct the \textit{aggregate} optimal context $\tilde{\bm x}_{v, t} = (\bm z_v, {\bm x}_{v, t})$. $\widehat F_{v, t}$ is obtained by solving:
\begin{equation}
\label{eqn:ridge_regression_multi}
\resizebox{0.95\linewidth}{!} 
{$\widehat F_{v, t} = \argmin_{f \in \mathcal H_{K}} \frac{1}{n_v(t)}\left(\sum_{(\tilde{\bm x}, y) \in \widetilde{\bm X}_{v, t}} \left(f(\tilde{\bm x}) - y\right)^2\right) + \lambda \lVert f \rVert_{\mathcal H_{K}}^2.$}
\end{equation}
Here, $\widetilde{\bm X}_{v, t} = (\tilde{\bm x}_i, y_i)_{i = 1}^{n_v(t)}$ denotes the $n_v(t)$ total action-reward pairs available at time $t$. Note that this comprises not just personal observations, but additional observations available via the messages received until that time. The solution to the above problem (\ref{eqn:ridge_regression_multi}) is given as:
\begin{equation}
    \widehat F_{v, t}(\tilde{\bm{x}}) = \bm \kappa_{v, t}(\tilde{\bm{x}})^\top \left(\bm K_{v, t} + \lambda \bm I\right)^{-1}\bm y_{v, t}.
\end{equation}
Here, $\bm \kappa_{v, t}(\tilde{\bm{x}}) = \left(K\left(\tilde{\bm x}, \tilde{\bm x}_i\right)\right)_{i=1}^{n_v(t)}$ denotes the vector of kernel values between the input vector $\bm x$ and all previously stored data by agent $v$, and similarly $\bm y_{v, t} = \left(y_{v, i}\right)_{i=1}^{n_v(t)}$ denotes the vector of rewards. The matrix $\bm K_{v, t}$ denotes the $n_v(t) \times n_v(t)$ matrix of kernel evaluations of every pair of samples $\tilde{\bm x}_i, \tilde{\bm{x}}_j \in \tilde{\bm X}_{v, t}$ possessed by agent $v$. To construct the sequence $(\beta_{v, t})_{t}$, following result motivates the upper confidence bound.
\begin{lemma}[\citet{chowdhury2017kernelized}]
\label{lemma:ucb_kernel}
Let $\widetilde{\mathcal X} \subset \mathbb R^d$, and $F : \widetilde{\mathcal X} \rightarrow \mathbb R$ be a member of the RKHS of real-valued functions on $\widetilde{\mathcal X}$ with kernel $K$, and RKHS norm bounded by $B$. Then, with probability at least $1-\delta$, the following holds for all $\tilde{\bm x} \in \widetilde{\mathcal X}$, and simultaneously for all $t \geq 1, v \in \mathcal V$:
\begin{equation*}\resizebox{0.95\linewidth}{!} 
{$ \Delta_{v, t}(\tilde{\bm x}) \leq \sigma_{v, t-1}^2(\tilde{\bm x})\left(B + R \sqrt{\ln\frac{\det\left(\lambda\bm I +{\bm K}_{v, t}\right)}{\delta^2} + 2\ln(V) }\right).$}
\end{equation*}
\end{lemma}
Where $\Delta_{v, t}(\tilde{\bm x}) = \left| F(\tilde{\bm x}) - \widehat F_{v, t}(\tilde{\bm x}) \right|$ and we denote $\sigma_{v, t-1}^2(\tilde{\bm x}) = K\left(\tilde{\bm x}, \tilde{\bm x}\right) - {\bm \kappa}_{v, t}(\tilde{\bm{x}})^\top \left({\bm K}_{v, t} + \lambda \bm I\right)^{-1}{\bm \kappa}_{v, t}(\tilde{\bm{x}})$ as the ``variance'' proxy for the UCB for brevity. This confidence bound is derived using the stronger, kernelized self-normalized concentration inequality from~\citet{chowdhury2017kernelized}, that holds \textit{simultaneously} for all $t \geq 1$, and hence prevents the second logarithmic term, in contrast to~\cite{srinivas2009gaussian} for continuum-armed $\mathcal D_{v, t}$.

\textbf{Controlling Drift via Clique Partitions}. The fundamental idea in controlling regret is to bound the per-round regret incurred by any agent by the UCB ``variance'' term $\sigma^2_{v, t-1}(\tilde{\bm x}_{v, t})$, and the algorithm attempts to bound $\sum_{v \in V} \sigma^2_{v, t-1}(\tilde{\bm x}_{v, t})$ by a quantity smaller than $ O(\sqrt{V})$ (i.e., improve over non-cooperative behavior). Our approach is to obtain this rate by partitioning $\mathcal G$ into $G$ subgraphs $\mathcal G'_1, ..., \mathcal G'_G$, and ensuring that the variance terms are similar for each agent within a subgraph for all $t$.

Our partitioning solution is a conservative one: let $\mathbf C$ be a clique covering of the $\gamma$-power of $\mathcal G$. For any clique $\mathcal C \in \mathbf C$, we restrict each agent $v \in \mathcal C$ to only accept observations from agents that belong to $\mathcal C$ as well. This ensures that at any $t$, any agent $v \in \mathcal C$ has an upper bound on $\sigma^2_{v, t-1}$ that depends only on $\mathcal C$. Therefore we can control the group regret within each clique $\mathcal C$, leading to a factor of $\sqrt{\bar{\chi}(\mathcal G_\gamma)}$ (instead of $V$) in the regret, where $\bar{\chi}(\cdot)$ is the clique number.

\begin{algorithm}[t!]
\caption{\textsc{COOP-KernelUCB}}
\label{alg:main_widthdelay}
\small
\begin{algorithmic}[1] 
\STATE \textbf{Input}: Graph $\mathcal G_\gamma$ with clique cover $\bm C_\gamma$, kernels $K_x(\cdot, \cdot), K_z(\cdot, \cdot)$, $\lambda$, explore param. $\eta$, buffers $\bm B_{v} = \phi$.
\FOR{For each iteration $t \in [T]$}
\FOR{For each agent $v \in V$}
\IF{$t = 1$}
\STATE ${\bm x}_{v, t} \leftarrow$\textsc{Random}$\left(D_{v, t}\right)$.
\ELSE
\STATE ${\bm x}_{v, t} \leftarrow \underset{\bm x \in D_{v, t}}{\argmax} \left(\hat f_{v, t}(\bm z_v, \bm x) + \frac{\eta}{\sqrt{\lambda}} \sigma_{v, t-1}(\bm z_v, \bm x)\right)$.
\ENDIF
\STATE $\tilde{\bm x}_{v, t} \leftarrow (\bm z_v, {\bm x}_{v, t}), y_{v, t} \leftarrow$\textsc{Pull}$\left(\tilde{\bm x}_{v, t}\right)$.
\IF{$t=1$}
\STATE $({\bm K}_{v, t})^{-1} \leftarrow 1/K(\tilde{\bm x}_{v, t}, \tilde{\bm x}_{v, t}) + \lambda$.
\STATE $\bm y_{v} \leftarrow [y_{v, 0}]$.
\STATE $\bm \kappa_v = (K(\cdot, \tilde{\bm x}_{v, t}))$.
\ELSE
\STATE $\bm B_{v} \leftarrow \bm B_{v} \cup \left(\tilde{\bm x}_{v, t}, y_{v,t}\right)$.
\ENDIF
\STATE $\bm m_{v, t} \leftarrow \left\langle t, v, \tilde{\bm x}_{v, t}, y_{v, t}\right\rangle$.
\STATE \textsc{SendMessage}$\left(\bm m_{v, t}\right)$.
\FOR{$\langle t', v', \tilde{\bm x}', y'\rangle$ in \textsc{RecvMessages}$(v, t)$}
\IF{$v' \in \textsc{Clique}(v, \mathbf C_\gamma)$}
\STATE $\bm B_{v} \leftarrow \bm B_{v} \cup \left(\tilde{\bm x}', y'\right)$.
\ENDIF
\ENDFOR
\FOR{$(\tilde{\bm x}', y') \in \bm B_{v}$}
\STATE $\bm y_{v} \leftarrow [\bm y_{v}, y']$.
\STATE $\bm \kappa_v = (\bm \kappa_v , K(\cdot, \tilde{\bm x}'))$.
\STATE $\bm K_{22} \leftarrow \left(K(\tilde{\bm x}', \tilde{\bm x}') + \lambda - (\bm \kappa_v)^\top ({\bm K}_{v, t})^{-1} \bm \kappa_v\right)^{-1}$.
\STATE $\bm K_{11} \leftarrow \left(({\bm K}_{v, t})^{-1} + \bm K_{22}({\bm K}_{v, t})^{-1}\bm \kappa_v(\bm \kappa_v)^\top ({\bm K}_{v, t})^{-1}\right)$.
\STATE $\bm K_{12} \leftarrow -\bm K_{22}({\bm K}_{v, t})^{-1}\bm \kappa^\tau_v$.
\STATE $\bm K_{21} \leftarrow -\bm K_{22}(\bm \kappa_v)^\top({\bm K}_{v, t})^{-1}$.
\STATE $({\bm K}_{v, t})^{-1} \leftarrow [\bm K_{11}, \bm K_{12}; \bm K_{21}, \bm K_{22}]$.
\ENDFOR
\STATE $\bm B_{v} = \phi$.
\STATE $\hat f_{v, t+1} \leftarrow \left(\bm \kappa_v\right)^\top({\bm K}_{v, t})^{-1}\bm y_{v}$.
\STATE $s_{v, \rho+1} \leftarrow \sqrt{K(\cdot, \cdot) - \left(\bm \kappa_v\right)^\top({\bm K}_{v, t})^{-1}\bm \kappa_v}$.
\ENDFOR
\ENDFOR
\end{algorithmic}
\end{algorithm}
\begin{remark}[Computational complexity]
As outlined in~\citet{valko2013finite}, it is possible to perform an $\mathcal O(1)$ update of the Gram matrix, via the Schur decomposition (l. 30-34, Algorithm~\ref{alg:main_widthdelay}). This update can also be applied when $K_z$ is unknown and approximated online, see Section~\ref{sec:approx}.
\end{remark}
Algorithm~\ref{alg:main_widthdelay} presents the \textsc{Coop-KernelUCB} algorithm with a tunable exploration parameter $\eta$, which can be different from the parameter $\beta_{v, t}$ used in the analysis, as is typical in this setting~\cite{gentile2014online, chowdhury2017kernelized}. We now present a regret bound for this algorithm.

\begin{theorem}[Group Regret under Delayed Communication]
Let $\mathbf C$ be a minimal clique covering of $\mathcal G_\gamma$. When $\mathcal D_{v, t}$ is continuum-armed, Algorithm~\ref{alg:main_widthdelay} incurs a per-agent average regret that satisfies, with probability at least $1-\delta$,
\begin{multline*}
\widehat{\mathcal R}(T) = \mathcal O\Bigg(\sqrt{\bar{\chi}(\mathcal G_\gamma)\cdot \frac{T}{V}}\bigg(R\cdot \widehat{\Upsilon}_T \\ + \sqrt{\widehat{\Upsilon}_T }\bigg(B + R\sqrt{2\log\frac{V\lambda}{\delta}}\bigg)\bigg)\Bigg).
\end{multline*}
$\text{Here } \widehat{\Upsilon}_T  = \max_{\mathcal C \in \bm C}\left[\log\det\left(\frac{1}{\lambda}{\bm K}_{\mathcal C, T} + \bm I\right)\right]$ is the overall information gain, and for any clique $\mathcal C \in \mathbf C$, the matrix $\bm K_{\mathcal C, T}$ is the Gram matrix formed by actions from all agents within $\mathcal C$ until time $T$, i.e. $(\tilde{\bm x}_{v, t})_{v \in \mathcal C, t \in [T]}$.
\end{theorem}
We first discuss the leading factors in the bound. Compared to single-agent bounds, a coarse approximation of our rate reveals an additional factor of $\mathcal O\left(\sqrt{\bar{\chi}(\mathcal G_\gamma)}\right)$. This factor arises from the delayed spread of information, and is equal to the minimum clique number of the $\gamma$ power graph of the communication network. When $\mathcal G$ is $\gamma$-complete (i.e., $\mathcal G_\gamma$ is complete), $\bar{\chi}(\mathcal G_\gamma) = 1$, providing us the best rate. An example topology when this condition is realized include $\gamma/2$-star graphs (one node at the center, and `spikes' of $\gamma/2$ nodes). Conversely, since we assume $\mathcal G$ is connected, in the worst-case graph, $\bar{\chi}(\mathcal G_\gamma) = \ceil{V/\gamma}$. in which case the regret bound is equivalent to that obtained when all agents run in isolation. This is achieved, for example, in a line graph. Now, we formalize the idea of heterogeneity among agents.
\begin{definition}[Heterogeneity]
In a multi-agent setting with $V$ agents and context vector space $\mathcal Z$ with kernel $K_z$, let $\bm K_z$ be the matrix of pairwise interactions, i.e., $\bm K_z = (K(\bm z_v, \bm z_{v'}))_{v, v' \in \mathcal V}$. Then, the corresponding \textbf{heterogeneity} $\Upsilon_z$ for this setting is defined as $\Upsilon_z = \text{rank}(\bm K_z)$. 
\end{definition}

For the composition considered in our paper, we can derive a regret bound in terms of $\Upsilon_{VT}^x$, i.e., the information gain from $VT$ actions $\bm x$ across agents and heterogeneity $\Upsilon_z$.
\begin{corollary}
When $K = K_z \odot K_x$, Algorithm~\ref{alg:main_widthdelay} incurs a per-agent average regret, w. p. at least $1-\delta$,
\begin{align*}
    \widehat{\mathcal R}(T) = \widetilde{\mathcal O}\Bigg(\Upsilon_z\cdot \Upsilon_{VT}\cdot\sqrt{\bar{\chi}(\mathcal G_\gamma)\cdot \frac{T}{V}\cdot\log\left(\frac{V\lambda}{\delta}\right)}\Bigg).
\end{align*}
\label{corr:fusion_effect}
\end{corollary}
\begin{remark}
The regret bound displays a smooth interaction between the network structure, communication delays and agent similarity\footnote{We demonstrate this smooth relationship for the product kernel, i.e. $ K = K_z \odot K_x$, however, alternate relationships are worth exploring. For more details and results on some forms of kernel compositions, see Appendix.}. Corollary~\ref{corr:fusion_effect} implies that the communication effectively acts as a ``mask'' on the underlying cooperative performance, which is controlled primarily by the proximity of the network contexts themselves. For instance, when network contexts are identical (\textit{fully-cooperative}), $\Upsilon_z = 1$, and then the network structure entirely determines the regret (via $\bar\chi(\mathcal G_\gamma)$). Conversely, if $\bm K_z$ is full-rank, then $\Upsilon_z = V$, and agents cannot leverage cooperation. In this case, no improvement can be obtained regardless of the density of $\mathcal G$.
\end{remark}

\textbf{\textit{Examples}}. We now provide a few examples to illustrate the problem setting. Consider the case when $K_x$ and $K_z$ are both linear. In this case, the algorithm can be understood as a weighted variant of the linear UCB algorithm: for an observation $\bm x_{v, t}$ from an agent $v$ to $v'$, the ``weighted'' observation is given by $(\bm z_v^\top \bm z_{v'}) \cdot \bm x_{v, t}^\top \bm x_{v, t}$, and hence the ``weight'' is  $(\bm z_v^\top \bm z_{v'})$. In an ideal implementation, the vectors $\bm z \in \mathbb S_{d-1}(1)$, i.e., the unit sphere in $d$ dimensions, such that $\bm z^\top \bm z = 1$, ensuring that personal observations are given a weight of 1. Alternatively, when both $K_z$ and $K_x$ are RBF kernels, we observe an additive effect, i.e., $K = \exp\left(-\frac{\lVert \bm z_v - \bm z_{v'} \rVert^2}{2\sigma_z^2} - \frac{\lVert \bm x_v - \bm x_{v'} \rVert^2}{2\sigma_x^2}\right) = \exp\left(-\frac{\lVert \hat{\bm x}_v - \hat{\bm x}_{v'}\rVert^2}{2}\right)$, where $\hat{\bm x} = \begin{bmatrix} \sfrac{\bm x}{\sigma_x} \\ \sfrac{\bm v}{\sigma_v}\end{bmatrix}$. Note that the additional factor incurred in comparison to single-agent learning is $\Upsilon_z = \mathcal O(d\ln V)$ for RBF $K_z$, and for any action or network kernel, the regret can be obtained via Remark~\ref{remark:single_agent}.

\subsection{Approximating Network Contexts}
\label{sec:approx}
The previous analysis assumes the availability of the underlying \textit{network context} vectors $\bm z_v$ for each agent (or at least oracle access to the kernel $K_z$), however, for many applications, this information is not available, and must be estimated from the contexts themselves. Our approach is based on \textit{kernel mean embeddings} \cite{blanchard2011generalizing, christmann2010universal, deshmukh2017multi}.

Consider the network context space $\mathcal Z$ to be the RKHS $\mathcal H_{K_x}$, and we assume that the contexts $\bm x_{v, t}$ for each agent are drawn from an underlying probability density $\mathcal P_v$. The idea is to use $\bm z_v$ as a representation of $\mathcal P_v$, so that we can (with an appropriate metric), use $K_z$ as a measure of ``similarity'' of the context distributions. For this, we look towards \textit{kernel mean embeddings} of the distributions $\mathcal P_v$ in the RKHS $\mathcal H_{K_x}$. This implies that the augmented context $\tilde{\bm x}_{v, t}$ at any time $t$ for any agent $v \in V$ is $\left(\Psi(\mathcal P_v), \bm x_{v, t}\right)$, where $\Psi(\mathcal P_v) = \mathbb E_{\bm x \sim \mathcal P_v} [\phi_x(\bm x)] = \mathbb E_{\bm x \sim \mathcal P_v} [K_x(\cdot, \bm x)]$ is the kernel mean embedding of $\mathcal P_v$ in $\mathcal H_{K_x}$.  Using this, we can define the kernel $K_z$ as follows.
\begin{equation}
    K_{z}\left(\Psi(\mathcal P_v), \Psi(\mathcal P_{v'})\right) = \exp\left(-\lVert \Psi(P_v) - \Psi(P_{v'})\rVert_2 / 2\sigma_z^2\right).
\end{equation}
We can estimate this from the available context via the \textit{empirical mean kernel embedding} $\widehat{\Psi}_t(\mathcal P_v) = \frac{1}{t}\sum_{i=1}^t K_x\left(\cdot, \bm x_{v, i}\right)$. Now, we can calculate the \textit{empirical kernel approximation} $\widehat{K}_{z, t}(\cdot, \cdot)$ at time $t$: 
\begin{align}
    \widehat{K}_{z, t}(\mathcal P_v, \mathcal P_{v'}) &= \exp\left(-\text{MMD}_{\mathcal H}(\widehat{\Psi}_t(\mathcal P_v), \widehat{\Psi}_t(\mathcal P_{v'})) / 2\sigma_z^2\right),
\end{align}
The empirical maximum mean discrepancy (MMD)~\cite{gretton2012kernel} is the measure employed to measure the divergence of the embeddings in $\mathcal H_{K_x}$, and is given by:
\begin{multline}
    \text{MMD}^2_{\mathcal H}(\widehat{\Psi}_t(\mathcal P_v), \widehat{\Psi}_t(\mathcal P_{v'})) =\sum_{\tau,\tau'}^{t, t} (K_x(\bm x_{v, \tau}, \bm x_{v, \tau'}) \\
    + K_x(\bm x_{v', \tau}, \bm x_{v', \tau'}) - 2K_x(\bm x_{v, \tau}, \bm x_{v', \tau'})).
\end{multline}
Our next result describes how the approximation (constructed from $t$ samples each) $\widehat{K}_t = \widehat{K}_{z, t} \odot K_x$ deviates from the true kernel $K = K_z \odot K_x$ under this model.
\begin{lemma}
For an RKHS $\mathcal H$, assume that $\lVert f \rVert_{\infty} \leq d$ for all $f \in \mathcal H$ with $\lVert f \rVert_{\mathcal H} \leq 1$. Then, the following is true with probability at least $1-\delta$ for all $\bm x_i, \bm x_j \in \widetilde{\mathcal X}$:
\begin{equation*}
\resizebox{\linewidth}{!} 
{
    $\left| \log \left(\frac{\widehat{K}_{z, t}\left(\bm x_i, \bm x_j\right)}{K_z\left(\bm x_i, \bm x_j\right)}\right)\right| \leq \frac{1}{\sigma_z^2}\left(\underset{\mathcal P \in \mathfrak P_{\mathcal X}}{\sup}\mathfrak R_t(\mathcal H, \mathcal P) + 2d\sqrt{\frac{1}{2t}\log\frac{1}{\delta}}\right).$}
\end{equation*}
Here $\mathfrak R_t(\mathcal H, \mathcal P_v)$ denotes the $t$-sample Rademacher average~\cite{bartlett2002rademacher} of $\mathcal H$ under $\mathcal P_v \in \mathfrak P_{\mathcal X}$.
\label{lem:rkhs}
\end{lemma}
Lemma~\ref{lem:rkhs} implies the \textit{consistency} of the empirical Kernel estimator, i.e., for any $v, v' \in \mathcal V$, $\widehat K_{z, t} \rightarrow K_z$ with probability 1 as $t \rightarrow \infty$. To obtain $K_{z}$ we can employ any other PSD kernel $K_{P}$ on $\mathcal X$ besides $K_x$ as well. 
\begin{remark}[Regret of Simultaneous Estimation]
At any instant, the empirical heterogeneity is locally controlled, i.e., for a clique cover $\mathbf C$ of $\mathcal G_\gamma$, $\Upsilon_z \leq \max_{\mathcal C \in \mathbf C} |\mathcal C|^2$. This follows directly from Theorem 2 of~\citet{krause2011contextual} and the fact that for any agent in a clique $\mathcal C$, the empirical kernel approximation only takes $\sfrac{1}{2}\left(|\mathcal C|\cdot(|\mathcal C| - 1)\right)$ distinct values at any instant. This implies that sparse network settings can easily be shown to benefit from cooperation (i.e., when $|\mathcal C| = \mathcal O(V^{\sfrac{1}{4}})$), but future work can address stronger controls on the group regret.
\end{remark}
\section{Extensions}
\label{sec:extensions}
\textbf{\textit{Fully-Cooperative Setting}}. In the fully cooperative setting, the \textit{network contexts} for all agents are identical, and each agent is solving the same contextual bandit problem. When the decision set is fixed (i.e. $D_{v, t} = D \subset \mathcal X$ for all $v, t$, \citet{cesa2019delay}) we can derive a variant of \textsc{Coop-KernelUCB} that provides optimal performance. The central idea of the algorithm is for centrally-positioned agents to essentially follow Algorithm~\ref{alg:main_widthdelay}, however, the agents that are positioned peripherally in $\mathcal G$ mimic the actions (obtained after the appropriate delay) of these centrally positioned agents. We partition the set of agents $\mathcal G$ into ``central'' and ``peripheral'' agents such that each peripheral agent is connected to at least one central agent in $\mathcal G_\gamma$. This algorithm is defined as \textsc{Dist-KernelUCB} as this algorithm is not decentralized. It just remains to define the partition, which we do after introducing some notation.

\begin{definition}
An independent set of a graph $\mathcal G = (\mathcal V, \mathcal E)$ is a set of vertices $\mathcal V' \subseteq \mathcal V$ such that no two vertices in $\mathcal V'$ are connected. A maximal independent set $\mathcal V^*$ is the largest independent set, and independence number $\alpha(\mathcal G) = |\mathcal V^*|$.
\end{definition}
We set the ``central'' agents $\mathcal V_C$ of $\mathcal G$ as the maximal weighted independent set of $\mathcal G_\gamma$ (where, for any node $v \in \mathcal V$, the weight $w_v =  N_\gamma(v)$), and set the complement $\mathcal V_P = \mathcal V \setminus \mathcal \mathcal V_C$ as the ``peripheral'' set. Each peripheral agent $p$ is assigned the central agent it is connected to (denoted as $\text{cent}(p)$), and in case any peripheral agent is connected to more than one central agent, we assign it to the central agent with maximum degree in $\mathcal G_\gamma$. The set of peripheral agents assigned to a central agent $c$ is denoted by $\pi(c)$. We can then make the following claims about regret incurred in this setting.
\begin{theorem}
$\mathcal D_{v, t}$ is continuum-armed, \textsc{Dist-KernelUCB} incurs a per-agent average regret that satisfies, with probability at least $1-\delta$, 
\begin{align*}
    \widehat{\mathcal R}(T) = \widetilde{\mathcal O}\Bigg(\Upsilon_z\cdot \Upsilon_{VT}\cdot\sqrt{\alpha(\mathcal G_\gamma)\cdot \frac{T}{V}\cdot\log\left(\frac{V\lambda}{\delta}\right)}\Bigg).
\end{align*}
Here, $\alpha(\mathcal G_\gamma)$ refers to the independence number of  $\mathcal G_\gamma$.
\label{thm:fully_coop}
\end{theorem}
The central concept utilized in this case is to partition the network in a manner that allows for a group of agents to make identical (albeit delayed) decisions. The regret analysis uses the property that the vertex cover of the elements of $V_C$ spans $\mathcal G_\gamma$, and one can bound the total regret by simply bounding the regret incurred by each ``central'' agent, since for any ``peripheral'' agent $v$, the regret incurred is only a constant larger than the correponding regret incurred by the ``central'' agent, i.e., $\mathcal R_v(T) \leq \widetilde{O}(\sqrt{\gamma}) + \mathcal R_{\text{cent}(v)}(T)$. The full proof and algorithm pseudocode are in the appendix.
\begin{remark}
In addition to the tighter average per-agent regret (since $\alpha(\mathcal G_\gamma) \leq \bar\chi(\mathcal G_\gamma)$), we can make a stronger claim about the \textit{individual} regret for any agent as well. Both these regret bounds match the rates mentioned for the context-free case in~\cite{cesa2019delay, bar2019individual}. Moreover, when $\gamma = 1$, then the bound on the group regret matches the lower bound shown in the nonstochastic case~\cite{cesa2019cooperative}. 
\end{remark}

\textbf{\textit{Eager Estimation}}. In addition to the algorithms mentioned above, we also consider a (potentially stronger) variant of \textsc{Coop-KernelUCB} that does not take delays into account at all, and simply updates its observation set as soon as it obtains any new information (from any communicating agent). Consequently, for any arbitrary $\mathcal G$ and $\gamma$, this can lead to significant \textit{drift} in the Gram matrices for any pair of agents, making this algorithm (dubbed \textsc{Eager-KernelUCB}) significantly more challenging to analyze. We defer the analysis therefore to future work and present empirical evaluations of this variant in this paper. This algorithm can be understood as Algorithm~\ref{alg:main_widthdelay} run with all observations (i.e., lines 20-22 in Algorithm~\ref{alg:main_widthdelay} are ignored), although we present the complete pseudocode in the appendix.

\section{Experiments}
The central aspect we wish to experimentally understand is the behavior of the algorithm with respect to network structures and delay in cooperative learning (alternatively, a detailed experimental comparison of single-agent KernelUCB under Gaussian noise can be found in~\cite{srinivas2009gaussian, krause2011contextual}), and hence our experimental benchmark setup focuses on these aspects as well. We conduct two major lines of experimentation, the first on synthetically generated random networks, and the second on real-world networks subsampled from the SNAP network datasets~\cite{leskovec2016snap}.
\begin{figure*}[t!]
  \includegraphics[width=\linewidth]{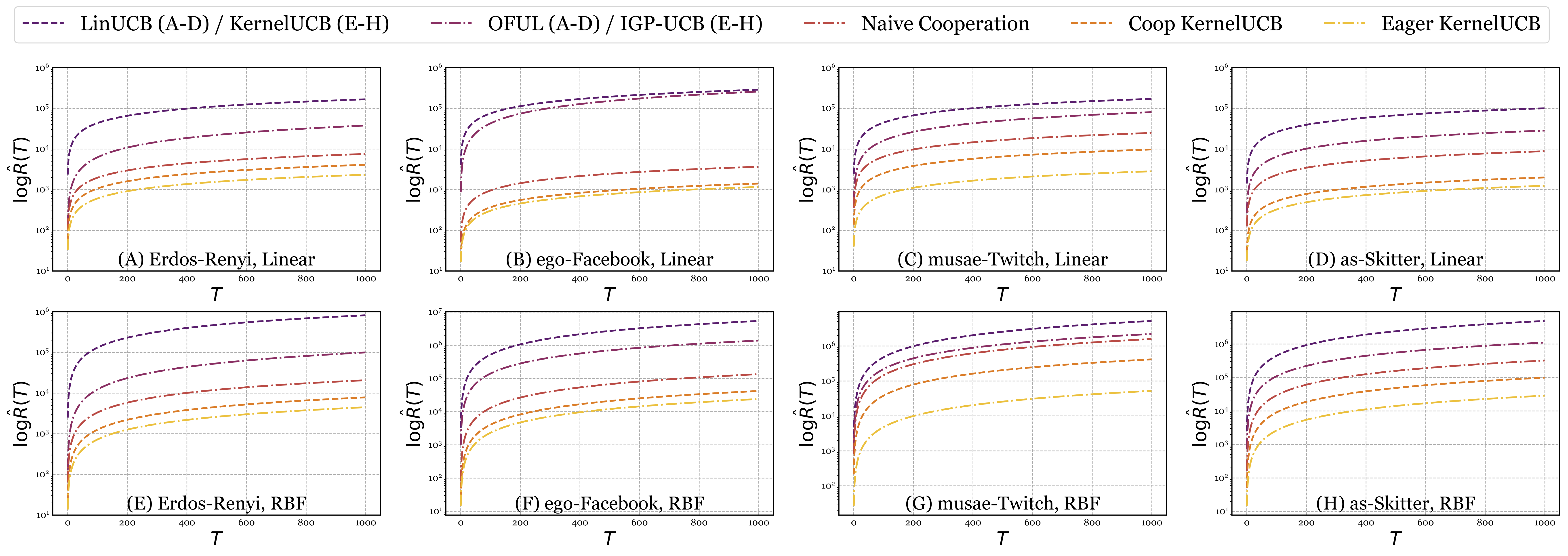}
  \caption{An experimental comparison of \textsc{COOP-KernelUCB} and its variants with benchmark techniques for contextual bandits. Each experiment is averaged over 100 trials. The top row denotes the linear kernel, and the bottom row denotes the RBF kernel.}
  \label{fig:boat1}
\end{figure*}

\textbf{\textit{Comparison Environments}}. We compare in two benchmark setups. In order to compare performance with linear methods, our first setup assumes $K_x$ is the linear kernel, and $K_z$ is a clustering of the agents given by the independent sets of the $\gamma^{th}$ power of the underlying connectivity graph $\mathcal G$ ($K_z$ not known to the algorithms \textit{a priori}, and $\gamma = \text{diameter}(\mathcal G)/2$). This is done to motivate the central application scenario where the network connectivity and task similarity are correlated. The second setup is where $K_x$ and $K_z$ are both randomly initialized Gaussian kernels (where $K_z$ is again unknown to our method). We run both setups on graphs of $V =200$ nodes, $D_{v, t} $ is a set of 8 randomly generated contexts for all $v \in \mathcal V, t \in T$ and dimensionality $d = 10$ for $\mathcal X$ and $\mathcal Z$ (for setup 2). For the kernel estimation task, we set $\sigma = 1$, and we set $\lambda = 1$.

\textbf{\textit{Network Structures}}. We run experiments on two network structures - (a) synthetic, randomly generated networks and (b) real-world networks. For the synthetic networks, we generate random connected Erdos-Renyi networks~\cite{erdHos1960evolution} of size $V = 200$ with $p=0.7$. For the synthetic networks, we subsample $V$ nodes and their corresponding edges (for $V = 200$) from the \texttt{ego-Facebook}, \texttt{musae-Twitch}, and \texttt{as-Skitter} networks, in order to represent a diverse set of networks found in social networks, peer-to-peer distribution and autonomous systems. 

\textbf{\textit{Benchmark Methods}}. In the linear setting, we compare against single-agent LinUCB~\cite{li2010contextual} (where every agent runs LinUCB independently), OFUL~\cite{abbasi2011improved} and \textsc{Coop-KernelUCB} and \textsc{Eager-KernelUCB}. In the kernel setting, we compare against single-agent KernelUCB~\cite{valko2013finite}, IGP-UCB~\cite{chowdhury2017kernelized}. Additionally, an important benchmark we compare against is \textit{Naive Cooperation}, where agents run IGP-UCB (kernel) and LinUCB (linear) but include observations from neighbors as their own (without reweighting).

\textbf{\textit{Results Summary}}.
Figure~\ref{fig:boat1} describes the regret achieved on each of the 4 benchmark networks for both linear and RBF (Gaussian) settings. Each plot is obtained after averaging the results for 100 trials, where the bandit contexts were refreshed every trial. We first highlight the general trend observed. Since the baseline techniques do not utilize cooperation at all, we expect them to provide a per-agent regret that scales linearly, instead of the $\mathcal O(1/\sqrt{V})$ dependence for our algorithms, which is obtained in our results as well. Among our algorithms, we see that \textsc{Coop-KernelUCB} and \textsc{Dist-KernelUCB} perform similarly for the Erdos-Renyi outperforms \textit{Naive Cooperation} in both the linear and kernel settings, which can be attributed to the fact that naive cooperation does not take agent similarities into account. We observe that \textsc{Eager-KernelUCB} consistently outperforms other algorithms, across all benchmark tasks. 

Our motivation for this algorithm stems from work in the delayed feedback regime for the stochastic (context-free) bandit~\cite{joulani2013online}, which suggests that incorporating observations \textit{as soon as} they are available can provide optimal regret. While it is challenging to derive a provably optimal algorithm in the contextual setting (and more challenging in the multi-agent case), we simply extended the ``as soon as'' heuristic in \textsc{Eager-KernelUCB}. The observed empirical regret suggests to us that \textsc{Eager-KernelUCB} obtains $O(\sqrt{\tfrac{T}{V}(\gamma + \alpha(\mathcal G_\gamma)))}$ regret, lower than the other variants of \textsc{Coop-KernelUCB}. The other variations between different graph families can be attributed to the difference in connectivity (instead of the kernel approximation).

\section{Discussion and Related Work}
This paper is inspired by and draws from concepts in several (often disparate) subfields within the literature. We discuss our contributions with respect to these areas sequentially.

\textbf{\textit{Cooperative Multi-Agent Learning}}. Cooperative bandit learning with delays has maintained the setting that all agents solve the \text{same} bandit problem (i.e., fully cooperative), which our work generalizes as a first step. In the nonstochastic (multi-armed) case (without delays), this problem was first studied in the work of~\citet{awerbuch2008online}, where they proposed an algorithm with a per-agent regret bound of $O(\sqrt{(1+KV^{-1})\ln T})$, which matches (up to logarithmic factors) our version of the bound in the same setting (with contexts). In the multi-armed case,~\cite{landgren2016distributed, landgren2016distributed2, landgren2018social, martinez2018decentralized} provide algorithms whose regret scales as a function of the graph Laplacian of $\mathcal G$, using a \textit{consensus} protocol~\cite{bracha1985asynchronous}. Our algorithms are based on a message-passing framework (i.e., \textsc{Local}), which maintains the same communication complexity, while providing significantly better regret guarantees. Moreover, we can express the consensus protocol as an instance (albeit restricted) of our algorithm, when $K_x(i, j) = \mu_i \bm 1\{i=j\}, \mu_i \in \mathbb R$ is a scaled simplex, and $K_z(i, j) = A_{ij}^{d(i, j)}$ is the power of the graph Laplacian. Algorithms for the nonstochastic non-contextual case with delays have been developed in~\cite{cesa2019delay, bar2019individual}, that propose algorithms with per-agent average regret scaling as  $\widetilde{\mathcal O}(\sqrt{\alpha(\mathcal G_\gamma)TV^{-1}})$ and individual regret (for agent $v \in V$) scaling as $\widetilde{\mathcal O}(\sqrt{(1+K|\pi(\text{cent}(v))|^{-1})T})$, which match the regret achieved by \textsc{Dist-KernelUCB} in the fully cooperative \textit{contextual} setting. A minimax regret bound for the nonstochastic context-free of $O(\sqrt{(\gamma + K)T})$ is also provided in~\cite{cesa2019delay}, improving on the work of~\citet{neu2010online}, which our work improves up to smaller network factors ($\sqrt{\bar\chi(\mathcal G_\gamma)}$). When we compare our regret bounds with the algorithm-agnostic delayed feedback regret bounds provided by~\cite{joulani2013online} for the single-agent case, we observe the same relationship.

\textbf{\textit{Leveraging Social Contexts}}. There has been extensive research in leveraging \textit{social} side-observations across the bandit literature.~\citet{cesa2013gang} provide an algorithm called \textit{GoB.Lin} that assumes an outer-product relationship between information flow in the network (via the graph Laplacian) and context information. This is exactly an instance of our framework, where the kernel $K(\bm x_i, \bm x_j)$ is described by $\tilde{\bm \phi}(\bm x_i)A_{\otimes}^{-1}\tilde{\bm \phi}(\bm x_j)$ (in their notation), extended to the (kernel) multi-agent case with delayed feedback. In their setting, our regret bounds match exactly those of GoB.Lin. The clustering formulation can also be seen as a variant of the kernel framework, where $K_z(\bm z_v, \bm z_{v'}) = 1$ if agents belong to the same cluster, and 0 otherwise. The clustering is not known~\textit{a priori}, however, and the work of~\cite{gentile2014online, gentile2017context, li2018online} provides algorithms with tight regret guarantees for this case (our kernel embedding technique is similar in this regard). Again, we highlight that while multi-agent decision-making has been studied in the social network case (with non-identical contexts)~\cite{korda2016distributed, wang2020distributed}, none, to our knowledge, consider general graph communication with delays. 

\textbf{\textit{Kernel Methods for Bandit Optimization}}. A theoretical treatment of kernelised bandit learning was first explored in the work of~\citet{valko2013finite}, which was built on the LinUCB~\cite{li2010contextual} and SupLinUCB~\cite{chu2011contextual}, that were in turn inspired by the early work of~\citet{auer2002finite}. Our work is an improvement on the single-agent algorithms provided by~\citet{valko2013finite} owing to the martingale inequality presented in~\cite{chowdhury2017kernelized}, who use their result to construct improved versions single-agent Gaussian Process bandit algorithms~\cite{krause2011contextual, srinivas2009gaussian}. Our work also improves on the multi-task framework introduced by~\cite{deshmukh2017multi} to the multi-agent setting with delays, along with a stronger regret bound, and an approximation guarantee for the \textit{kernel mean embedding} approach to estimate task similarity. Recent results~\cite{calandriello2019gaussian, janz2020bandit} on bandit optimization for certain kernel families can certainly be used to construct algorithmic variants with stronger guarantees on the context kernel $K_x$.
\section{Conclusion}
In this paper we presented \textsc{Coop-KernelUCB}, an kernelized algorithm for decentralized, multi-agent cooperative contextual bandits and proved regret bounds of $\widetilde{O}(\sqrt{\bar\chi(\mathcal G_\gamma)T/V})$ on the average pseudo-regret, and supported our theoretical developments with experimental performance. However, there are several aspects of the kernelised cooperative bandit problem that are left as open problems. An interesting first direction is to establish suitable \textit{lower bounds} on the group regret in cooperative decision-making with delays. An $\Omega(\sqrt{VT})$ lower bound~\cite{bubeck2012regret} can be derived for a single-agent playing $VT$ trials sequentially, however each of the $v \in V$ agents can greatly reduce their uncertainty at every trial when cooperating, hence understanding the limits of cooperation is an interesting endeavor. Next, we presented a variant~\textsc{Eager-KernelUCB} of our algorithm that does not attempt to control the drift between the agent Gram matrices, that outperforms our main algorithm. We conjecture that a regret guarantee of the order $\widetilde{O}(\sqrt{(\gamma + \alpha(\mathcal G_\gamma))T/V})$ exists for this algorithm in the linear case, and proving this under suitable assumptions is an interesting future direction as well. Finally, extending this line of research into the Bayesian case is also worth exploring. 
\appendix
\onecolumn
\section{Regret bound for \textsc{Coop-KernelUCB}}
We first state a few existing results that we will utilize in the proof.
\begin{theorem}[Theorem 2 of~\cite{chowdhury2017kernelized}]
Let $D \subset \mathbb R^d$, and $f : D \rightarrow \mathbb R$ be a member of the RKHS of real-valued functions on $D$ with kernel $K$, and RKHS norm bounded by $B$. Then, with probability at least $1-\delta$, the following holds for all $\bm x \in D$, and $t \geq 1$:
\begin{equation*}
    \left| f(\bm x) - \hat{f}_t(\bm x)\right| \leq s_t(\bm x)\left(B + R \sqrt{2\ln\frac{\sqrt{\det\left((1 +\eta)\bm I_t +\bm K_t\right)}}{\delta}}\right)
\end{equation*}
\label{thm:ucb_kernel_basic}
\end{theorem}
\begin{corollary}
\label{lemma:appx_ucb_kernel}
Let $\widetilde{\mathcal X} \subset \mathbb R^d$, and $F : \widetilde{\mathcal X} \rightarrow \mathbb R$ be a member of the RKHS of real-valued functions on $\widetilde{\mathcal X}$ with kernel $K$, and RKHS norm bounded by $B$. Then, with probability at least $1-\delta$, the following holds for all $\tilde{\bm x} \in \widetilde{\mathcal X}$, and simultaneously for all $t \geq 1, v \in \mathcal V$:
\begin{equation*}
\Delta_{v, t}(\tilde{\bm x}) \leq \sigma_{v, t-1}^2(\tilde{\bm x})\left(B + R \sqrt{\ln\frac{\det\left(\lambda\bm I +{\bm K}_{v, t}\right)}{\delta^2} + 2\ln(V) }\right).
\end{equation*}
\end{corollary}
\begin{proof}
This follows from Theorem~\ref{thm:ucb_kernel_basic} with probability $\delta/V$ for each agent $v \in \mathcal V$, and replacing $\lambda  = 1 + \eta$.` 
\end{proof}

\begin{theorem}[Theorem 2.1 of~\cite{zischur}, Characterization of Schur Decomposition]
Let $A$ be a Hermitian matrix given by
\begin{equation*}
A = 
\begin{pmatrix}
A_{11} & A_{12} & A_{13} \\
A_{21} & A_{22} & A_{23} \\
A_{31} & A_{32} & A_{33} \\
\end{pmatrix}, \text{then, }
    A_{33} - A_{32}A_{22}^{-1}A_{23} \geq A_{33} - \left(A_{31}, A_{32}\right)\begin{pmatrix}
A_{11} & A_{12} \\
A_{21} & A_{22} \\
\end{pmatrix}^{-1}\begin{pmatrix}
A_{13} \\
A_{23} \\
\end{pmatrix}.
\end{equation*}
\label{thm:schur}
\end{theorem}
The central observation in the regret bound is the control of the ``variance'' terms in each clique directly in terms of the corresponding clique Gram matrix. We describe this result in the following lemma. 
\begin{lemma}[Per-Clique Variance Bound]
Let $\mathcal C$ be a clique in $\mathcal G_\gamma$ and the clique Gram matrix $\bm K_{\mathcal C, T}$ be given by:
\begin{align*}
    {\bm K}_{\mathcal C, T} = \begin{pmatrix}
K(\tilde{\bm x}_{1, 1}, \tilde{\bm x}_{1, 1}) & ... & K(\tilde{\bm x}_{1, 1}, \tilde{\bm x}_{|\mathcal C|, T}) \\
\vdots & \ddots & \vdots \\
K(\tilde{\bm x}_{|\mathcal C|, T}, \tilde{\bm x}_{1, 1}) & \ldots & K(\tilde{\bm x}_{|\mathcal C|, T}, \tilde{\bm x}_{|\mathcal C|, T}) \\
\end{pmatrix}.
\end{align*}Then, for any $T\geq\gamma$,
\begin{align*}
    \sum_{t=\gamma}^T \sum_{v \in \mathcal C} \sigma^2_{v, t-1}(\tilde{\bm x}_{v, t}) \leq \gamma|\mathcal C|B + \max(1, \frac{1}{\lambda})\log\det\left(\frac{1}{\lambda}\bm K_{\mathcal C, T} + \bm I\right).
\end{align*}
\label{lem:s_bound1}
\end{lemma}
\begin{proof}
Consider a hypothetical agent that pulls arms in a round-robin fashion for all agents in $\mathcal C$, i.e., let the agents within the clique $\mathcal C$ be indexed (without loss of generality) as $1, 2, ..., |\mathcal C|$, and the agent pulls arms $\tilde{\bm x}_{1, 1}, \tilde{\bm x}_{2, 1}, ..., \tilde{\bm x}_{1, 2}, \tilde{\bm x}_{2, 2}, ..., \tilde{\bm x}_{|\mathcal C|-1, T}, \tilde{\bm x}_{|\mathcal C|, T}$. Therefore, the agent will pull a total of $|\mathcal C|T$ arms. At any time $t \in[|\mathcal C|T]$, let the corresponding \textsc{KernelUCB} parameters for this agent be given by:
\begin{align}
    {\bm K}_{\mathcal C, t} = \begin{pmatrix}
K(\tilde{\bm x}_{1, 1}, \tilde{\bm x}_{1, 1}) & ... & K(\tilde{\bm x}_{1, 1}, \tilde{\bm x}_{t \text{ mod } |\mathcal C|, \floor{t/|\mathcal C|}}) \\
\vdots & \ddots & \vdots \\
K(\tilde{\bm x}_{t \text{ mod } |\mathcal C|, t}, \tilde{\bm x}_{1, 1}) & \ldots & K(\tilde{\bm x}_{t \text{ mod } |\mathcal C|, \floor{t/|\mathcal C|}}, \tilde{\bm x}_{t \text{ mod } |\mathcal C|, \floor{t/|\mathcal C|}}) \\
\end{pmatrix}, \bm\kappa_{\mathcal C, t}(\bm x) = \left[K(\bm x, \tilde{\bm x}_{1, 1}), \ldots, K(\bm x, \tilde{\bm x}_{t \text{ mod }|\mathcal C|, \floor{t/|\mathcal C|}})\right].
\end{align}
Consider the variance functional for any agent $v \in \mathcal C$ at time $t \in \sigma_\tau$:
\begin{align}
    \sigma^2_{v, t-1}(\tilde{\bm x}_{v, t}) &= K\left(\tilde{\bm x}_{v, t}, \tilde{\bm x}_{v, t}\right) - {\bm \kappa}_{v, t}(\tilde{\bm x}_{v, t})^\top \left({\bm K}_{v, t} + \lambda \bm I\right)^{-1}{\bm \kappa}_{v, t}(\tilde{\bm x}_{v, t}) \\ \intertext{Let $\tau$ be the instance at which the round-robin agent pulls arm $\tilde{\bm x}_{v, t-\gamma}$. By Theorem~\ref{thm:schur}, we have for $t\geq\gamma$,}
    &\leq K\left(\tilde{\bm x}_{v, t}, \tilde{\bm x}_{v, t}\right) - {\bm \kappa}_{\mathcal C,\tau}(\tilde{\bm x}_{v, t-\gamma})^\top \left({\bm K}_{\mathcal C,\tau} + \lambda \bm I\right)^{-1}{\bm \kappa}_{\mathcal C,\tau}(\tilde{\bm x}_{v, t-\gamma}) \\
    &\leq K\left(\tilde{\bm x}_{v, t}, \tilde{\bm x}_{v, t}\right) - K\left(\tilde{\bm x}_{v, t-\gamma}, \tilde{\bm x}_{v, t-\gamma}\right) + K\left(\tilde{\bm x}_{v, t-\gamma}, \tilde{\bm x}_{v, t-\gamma}\right) - {\bm \kappa}_{\mathcal C,\tau}(\tilde{\bm x}_{v, t-\gamma})^\top \left({\bm K}_{\mathcal C,\tau} + \lambda \bm I\right)^{-1}{\bm \kappa}_{\mathcal C,\tau}(\tilde{\bm x}_{v, t-\gamma}). \\ \intertext{Let $\sigma^2_{\mathcal C, \tau} = K\left(\tilde{\bm x}_{v, t-\gamma}, \tilde{\bm x}_{v, t-\gamma}\right) - {\bm \kappa}_{\mathcal C,\tau}(\tilde{\bm x}_{v, t-\gamma})^\top \left({\bm K}_{\mathcal C,\tau} + \lambda \bm I\right)^{-1}{\bm \kappa}_{\mathcal C,\tau}(\tilde{\bm x}_{v, t-\gamma})$. Summing up over all $v \in \mathcal C$ and $t \geq \gamma$:}
    \sum_{t=\gamma}^T \sum_{v \in \mathcal C} \sigma^2_{v, t-1}(\tilde{\bm x}_{v, t}) &= \sum_{t'=T-\gamma}^T\sum_{v \in \mathcal C} K\left(\tilde{\bm x}_{v, t'}, \tilde{\bm x}_{v, t'}\right) + \sum_{\tau = 1}^{|\mathcal C|(T-\gamma)} \sigma^2_{\mathcal C,\tau} \leq \gamma|\mathcal C|B + \log\left(\prod_{\tau=1}^{|\mathcal C|T} (1+\sigma^2_{\mathcal C, \tau})\right).
\end{align}
Here the inequality follows since the kernel $K$ is bounded by $B$ and $\sigma^2_{\mathcal C,\tau} \leq \log(1+\sigma^2_{\mathcal C, \tau})$. Lemma 7 of~\cite{deshmukh2017multi} provides the following relationship for sequential pulls $\tilde{\bm x}_{v, t}, t \in [T]$ and their associated variance terms $\sigma^2_{v, t-1}(\tilde{\bm x}_{v, t})$ :
\begin{equation}
    \prod_{t \in [T]} (1+ \sigma^2_{v, t-1}(\tilde{\bm x}_{v, t})) = \frac{\det({\bm K}_{\mathcal C, T} + \lambda I)}{\lambda^{|\mathcal C|T+1}} = \det(\frac{1}{\lambda}{\bm K}_{\mathcal C, T} + \lambda I).
\end{equation}
This result is obtained using the determinant identity of the Schur decomposition provided by~\cite{zischur}. Applying this result to ${\bm K}_{\mathcal C,T}$  and variance terms $\sigma^2_{\mathcal C, \tau}$ gives us the final result (since the round-robin agent pulls arms sequentially). 
\end{proof}
Armed with this result we can now prove the regret bound.
\begin{theorem}[Group Regret under Delayed Communication]
Let $\mathbf C$ be a minimal clique covering of $\mathcal G_\gamma$. When $\mathcal D_{v, t}$ is continuum-armed, \textsc{Coop-KernelUCB} incurs a per-agent average regret that satisfies, with probability at least $1-\delta$,
\begin{equation*}
\widehat{\mathcal R}(T) = \mathcal O\left(\sqrt{\bar{\chi}(\mathcal G_\gamma)\cdot \frac{T}{V}}\left(R\cdot \widehat{\Upsilon}_T + \sqrt{\widehat{\Upsilon}_T }\left(B + R\sqrt{2\log\frac{V\lambda}{\delta}}\right)\right)\right).
\end{equation*}
$\text{Here } \widehat{\Upsilon}_T  = \max_{\mathcal C \in \bm C}\left[\log\det\left(\frac{1}{\lambda}{\bm K}_{\mathcal C, T} + \bm I\right)\right]$ is the overall information gain, and for any clique $\mathcal C \in \mathbf C$, the matrix $\bm K_{\mathcal C, T}$ is the Gram matrix formed by actions from all agents within $\mathcal C$ until time $T$, i.e. $(\tilde{\bm x}_{v, t})_{v \in \mathcal C, t \in [T]}$.
\end{theorem}
\begin{proof}
Consider the group pseudoregret at any instant $T$.
\begin{align}
  R_{\mathcal G}(T) &= \sum_{v \in \mathcal G} \left(\sum_{t=1}^T r_{v, t}\right) \intertext{Let us examine the individual regret $r_{v, t}$ of agent $v \in V$ at time $t$. From Theorem~\ref{lemma:ucb_kernel} and \textsc{Coop-KernelUCB}, we know that, for each agent $v \in V$, $  \beta_{v, t} \sigma_{v, t-1}(\tilde{\bm x}_{v, t}) + \hat{f}_{v, t}\left(\tilde{\bm x}_{v, t}\right) \geq  \beta_{v, t} \sigma_{v, t-1}(\tilde{\bm x}^*_{v, t}) + \hat{f}_{v, t}\left(\tilde{\bm x}^*_{v, t}\right),  f_v(\tilde{\bm x}^*_{v, t}) \leq \beta_{v, t} \sigma_{v, t-1}(\tilde{\bm x}^*_{v, t}) + \hat{f}_{v, t}\left(\tilde{\bm x}^*_{v, t}\right)$ and $ 
  \hat f_v(\tilde{\bm x}_{v, t}) \leq \beta_{v, t} \sigma_{v, t-1}(\tilde{\bm x}_{v, t}) + f_v\left(\tilde{\bm x}_{v, t}\right) $. Therefore for all $t \geq 1$ with probability at least $1-\delta$,}
    r_{v, t} &= f_v(\tilde{\bm x}^*_{v, t}) - f_v(\tilde{\bm x}_{v, t}) \\
    &\leq \beta_{v, t} \sigma_{v, t-1}(\tilde{\bm x}_{v, t}) + \hat{f}_{v, t}\left(\tilde{\bm x}_{v, t}\right) - f_{v}(\tilde{\bm x}_{v, t}) \\
    &\leq 2\beta_{v, t} \sigma_{v, t-1}(\tilde{\bm x}_{v, t}). \intertext{Therefore, for agent $v$, we have (since $\beta_{v, t} > \beta_{v, t-1}$~\cite{auer2002finite}),}
    \sum_{t=1}^T r_{v, t} &\leq 2\beta_{v, T} \sum_{t=1}^T \sigma_{v, t-1}(\tilde{\bm x}_{v, t}) \leq 2\gamma\sqrt{B}\beta_{v, \gamma} +  2\beta_{v, T} \sum_{t=\gamma}^T \sigma_{v, t-1}(\tilde{\bm x}_{v, t})
\end{align}
The second inequality follows from the fact that for all $t\leq \gamma, \beta_{v, t} \leq \beta_{v, \gamma}$ and that for all $v, t$, $\sigma_{v, t-1}(\tilde{\bm x}_{v, t}) \leq \sqrt{B}$. We can now sum up the second term for the entire group of agents. Setting $\beta^*_T = \max_{v \in V} \beta_{v, T}$, we get,
\begin{align}
    \sum_{t=\gamma}^T\sum_{v \in V} r_{v, t} &\leq 2\beta^*_T \left(\sum_{t=\gamma}^T\sum_{v \in V} \sigma_{v, t-1}\left(\tilde{\bm x}_{v, t}\right)\right) \\
    &\leq 2\beta^*_T \sqrt{V(T-\gamma)\left(\sum_{t=\gamma}^T\sum_{v \in V} \sigma^2_{v, t-1}\left(\tilde{\bm x}_{v, t}\right)\right)} \\
    &\leq 2\beta^*_T \sqrt{V(T-\gamma)\sum_{\mathcal C \in \mathbf C_\gamma}\left(\sum_{t=\gamma}^T\sum_{v \in \mathcal C} \sigma^2_{v, t-1}\left(\tilde{\bm x}_{v, t}\right)\right)} \\
    &\stackrel{(a)}{\leq} 2\beta^*_T \sqrt{V(T-\gamma)\sum_{\mathcal C \in \mathbf C_\gamma}\left(\gamma|\mathcal C|B + \max(1, \frac{1}{\lambda})\log\left(\frac{\det(\bm K_{\mathcal C, T} + \lambda\bm I)}{\lambda^{|\mathcal C|T +1}}\right)\right)} \\
    &\leq 2\beta^*_T \sqrt{V(T-\gamma)\cdot\bar\chi(\mathcal G_\gamma)\cdot\max(1, \frac{1}{\lambda})\left( \gamma B V+ \max_{\mathcal C \in \mathbf C}\left(\log\det(\bm K_{\mathcal C, T} + \lambda\bm I)\right)\right)} \\
    &\leq \beta^*_T  \cdot \mathcal O\left(\sqrt{\bar\chi(\mathcal G_\gamma)\cdot VT\cdot \widehat{\Upsilon}_T}\right).
\end{align}
Here, $(a)$ follows from Lemma~\ref{lem:s_bound1}. Now, from the definition of $\beta_{v, T}$ (Lemma~\ref{lemma:ucb_kernel}), we know that, for all $v \in V$ (where $v$ belongs to clique $\mathcal C$),
\begin{align}
    \beta_{v, T} &= B + R\sqrt{\lambda^{-1}}\sqrt{\log\left(\det\left({\bm K}_{v, T} + \lambda \bm I\right)\right)+\log\frac{2V}{\delta}} \\
    &\leq B + R\sqrt{\lambda^{-1}}\sqrt{\log\left(\det\left({\bm K}_{\mathcal C, T} + \lambda \bm I\right)\right)+\log\frac{2V}{\delta}} \\
    &\leq B + R\sqrt{\lambda^{-1}}\sqrt{\widehat{\Upsilon}_T+\log\frac{2V\lambda}{\delta}} \\
    \therefore \beta^*_T &= B + R\sqrt{\lambda^{-1}}\sqrt{\widehat{\Upsilon}_T+\log\frac{2V\lambda}{\delta}} \\
    &= \mathcal O\left(B + R\sqrt{\widehat{\Upsilon}_T + \log\frac{2V\lambda}{\delta}}\right).
\end{align}
Using this result in the earlier derivation, and then averaging over the number of agents $V$ gives us the final result.
\end{proof}
\subsection{Composition-Dependent Regret Bound}
Here we provide a proof for Corollary 1 from the main paper.
\begin{lemma}
Let $\Upsilon_z = \text{rk}\left(\bm K_{z}\right)$, where $\bm K_z = (K_z(\bm z_v, \bm z_{v'}))_{v, v' \in \mathcal V}$. When $K = K_z \odot K_x$, $ \widehat{\Upsilon}_T = 2\Upsilon_z \left(\Upsilon_T^x + \log(T)\right)$. When $K = K_z \oplus K_x$, $\widehat{\Upsilon}_T = 2\left(\Upsilon_z\log(T) + \Upsilon_T^x\right).$
\label{lem:rank_decomp}
\end{lemma}
\begin{proof}
We first note that $\widehat{\Upsilon}_T \leq \log\det\left(\frac{1}{\lambda}\bm K_T + \bm I\right)$, where $\bm K_T = (K(\tilde{\bm x}_{v, t}, \tilde{\bm x}_{v', t'}))_{v, v' \in \mathcal V, t, t' \in [T]}$. Furthermore, note that (a) $\text{rk}(\bm K^z_T) = \text{rk}(\bm K_z) $, where $\bm K^z_T = (K_z({\bm z}_{v, t}, {\bm z}_{v', t'}))_{v, v' \in \mathcal V, t, t' \in [T]}$, since $\bm K^z_T$ is composed entirely by tiling $T^2$ copies of $\bm K_z$. Now, to prove the first part, we simply use Theorem 2  of~\cite{krause2011contextual} on $\log\det\left(\frac{1}{\lambda}\bm K_T + \bm I\right)$. For the second part, we apply Theorem 3 of~\cite{krause2011contextual}.
\end{proof}
\begin{corollary}
When $K = K_z \odot K_x$, Algorithm 1 incurs a per-agent average regret, with probability at least $1-\delta$,
\begin{align*}
    \widehat{\mathcal R}(T) = \widetilde{\mathcal O}\Bigg(\Upsilon_z\cdot \Upsilon_{VT}\cdot\sqrt{\bar{\chi}(\mathcal G_\gamma)\cdot \frac{T}{V}\cdot\log\left(\frac{V\lambda}{\delta}\right)}\Bigg).
\end{align*}
\label{corr:fusion_effect}
\end{corollary}
\begin{proof}
This follows directly from Lemma~\ref{lem:rank_decomp} and Theorem 1.
\end{proof}

\section{Simultaneous Estimation of Contexts}
\subsection{Proof of Lemma 2}
\begin{proof}
We first state a concentration result for the kernel mean embedding obtained by~\citet{smola2007hilbert}. 
\begin{lemma}[\citet{smola2007hilbert}]
For an RKHS $\mathcal H$, assume that $\lVert f \rVert_{\infty} \leq d$ for all $f \in \mathcal H$ with $\lVert f \rVert_{\mathcal H} \leq 1$. Then, the following is true with probability at least $1-\delta$ for any $P_v \in \mathcal P_{\mathcal X}$:
\begin{equation*}
    \lVert \Psi(P_v) - \widehat{\Psi}_T(P_v) \rVert \leq 2\mathfrak R_T(\mathcal H, \mathcal P_{\mathcal X}) + d\sqrt{\frac{1}{T}\log(1/\delta)}.
\end{equation*}
\label{lem:smola_kme}
\end{lemma}
We now begin the proof for Lemma 2 by analysing the absolute log-ratio of the estimated kernel and true kernel at any time instant $T$. Consider two samples $\bm x_i, \bm x_j \in \widetilde{\mathcal X}$ at any instant $t$.
\begin{align}
    \left|\log\left(\frac{\widehat{K}_t(\bm x_i, \bm x_j)}{ K(\bm x_i, \bm x_j)}\right)\right| &= \frac{1}{2\sigma^2}\left| \lVert \Psi(P_i) - \Psi(P_j) \rVert - \lVert \widehat{\Psi}_T(P_i) - \widehat{\Psi}_T(P_j) \rVert \right| \\
    &\leq \frac{1}{2\sigma^2}\left\lVert \Psi(P_i) - \widehat{\Psi}_T(P_i) - \Psi(P_j)  + \widehat{\Psi}_T(P_j) \right\rVert \\
    &\leq \frac{1}{2\sigma^2}\left(\left\lVert \Psi(P_i) - \widehat{\Psi}_T(P_i) \right\rVert + \left\lVert \Psi(P_j)  - \widehat{\Psi}_T(P_j) \right\rVert\right).
\end{align}
Here, the first inequality is obtained via the reverse triangle inequality, and the second is obtained by Cauchy-Schwarz. Applying Lemma~\ref{lem:smola_kme} with probability $\delta/2$ on each term in the RHS, and replacing the Rademacher average for a specific $P$ with the $\sup$ completes the proof.
\end{proof}
\section{Proof of Theorem 2 of the Main Paper}
We begin with a few observations. Let the independent set used be given by $\mathcal V^* \subset \mathcal V$. For any agent $v \in \mathcal V \setminus \mathcal V^*$, let $c(v)$ denote the corresponding ``center'' agent that $v$ will mimic. Then, we first notice that for any $t \geq d(v, c(v))$, $\bm x_{v, t} = \bm x_{c(v), t-d(v, c(v))}$. We will continue with the notation used in the proof for Theorem 1. 
\begin{lemma}
Let $v \in \mathcal V^*$ be a ``center'' agent, and $N_\gamma(v)$ denote its gamma neighborhood (including itself). Without loss of generality, consider an ordering $1, 2, ..., |N_\gamma(v)|$ over the agents in $N_\gamma(v)$. Now, we define the neigborhood Gram matrix $\bm K_{v, T}$ as:
\begin{align*}
    {\bm K}_{v, T} = \begin{pmatrix}
K(\tilde{\bm x}_{1, 1}, \tilde{\bm x}_{1, 1}) & ... & K(\tilde{\bm x}_{1, 1}, \tilde{\bm x}_{|N_{\gamma}(v)|, T}) \\
\vdots & \ddots & \vdots \\
K(\tilde{\bm x}_{|N_{\gamma}(v)|, T}, \tilde{\bm x}_{1, 1}) & \ldots & K(\tilde{\bm x}_{|N_{\gamma}(v)|, T}, \tilde{\bm x}_{|N_{\gamma}(v)|, T}) \\
\end{pmatrix}.
\end{align*}
Assume all agents $\mathcal V$ follow \textsc{Dist-KernelUCB}. Then, for any agent $v \in V^*$ and for any $T\geq\gamma$,
\begin{align*}
    \sum_{t=\gamma}^T \sum_{v' \in \mathcal N_\gamma(v)} \sigma^2_{v, t-1}(\tilde{\bm x}_{v', t}) \leq B\gamma|N_\gamma(v)| + \max(1, \frac{1}{\lambda})\log\det\left(\frac{1}{\lambda}\bm K_{\mathcal C, T} + \bm I\right).
\end{align*}
\label{lem:s_bound2}
\end{lemma}
\begin{proof}
The proof is obtained in a manner similar to Lemma~\ref{lem:s_bound1}, with the trivial modification that each agent $v \in \mathcal V^*$ considers observations from its entire neighborhood $N_\gamma(v)$ and not just its parent clique.
\end{proof}
\begin{theorem}
$\mathcal D_{v, t}$ is continuum-armed, \textsc{Dist-KernelUCB} incurs a per-agent average regret that satisfies, with probability at least $1-\delta$, 
\begin{align*}
    \widehat{\mathcal R}(T) = \widetilde{\mathcal O}\Bigg(\Upsilon_z\cdot \Upsilon_{VT}\cdot\sqrt{\alpha(\mathcal G_\gamma)\cdot \frac{T}{V}\cdot\log\left(\frac{V\lambda}{\delta}\right)}\Bigg).
\end{align*}
Here, $\alpha(\mathcal G_\gamma)$ refers to the independence number of  $\mathcal G_\gamma$.
\label{thm:fully_coop}
\end{theorem}
\begin{proof}
Consider the group pseudoregret at any instant $T$.
\begin{align}
  \mathcal R(T) &= \sum_{v \in \mathcal G} \left(\sum_{t=1}^T r_{v, t}\right) \intertext{Let us examine the individual regret $r_{v, t}$ of agent $v \in V$ at time $t$. From Theorem~\ref{lemma:ucb_kernel} and \textsc{Coop-KernelUCB}, we know that, for each agent $v \in V$, $  \beta_{v, t} \sigma_{v, t-1}(\tilde{\bm x}_{v, t}) + \hat{f}_{v, t}\left(\tilde{\bm x}_{v, t}\right) \geq  \beta_{v, t} \sigma_{v, t-1}(\tilde{\bm x}^*_{v, t}) + \hat{f}_{v, t}\left(\tilde{\bm x}^*_{v, t}\right),  f_v(\tilde{\bm x}^*_{v, t}) \leq \beta_{v, t} \sigma_{v, t-1}(\tilde{\bm x}^*_{v, t}) + \hat{f}_{v, t}\left(\tilde{\bm x}^*_{v, t}\right)$ and $ 
  \hat f_v(\tilde{\bm x}_{v, t}) \leq \beta_{v, t} \sigma_{v, t-1}(\tilde{\bm x}_{v, t}) + f_v\left(\tilde{\bm x}_{v, t}\right) $. Therefore for all $t \geq 1$ with probability at least $1-\delta$,}
    r_{v, t} &= f_v(\tilde{\bm x}^*_{v, t}) - f_v(\tilde{\bm x}_{v, t}) \\
    &\leq \beta_{v, t} \sigma_{v, t-1}(\tilde{\bm x}_{v, t}) + \hat{f}_{v, t}\left(\tilde{\bm x}_{v, t}\right) - f_{v}(\tilde{\bm x}_{v, t}) \\
    &\leq 2\beta_{v, t} \sigma_{v, t-1}(\tilde{\bm x}_{v, t}). \intertext{Therefore, for agent $v$, we have (since $\beta_{v, t} > \beta_{v, t-1}$~\cite{auer2002finite}),}
    \sum_{t=1}^T r_{v, t} &\leq 2\beta_{v, T} \sum_{t=1}^T \sigma_{v, t-1}(\tilde{\bm x}_{v, t}) \leq 2\gamma\sqrt{B}\beta_{v, \gamma} +  2\beta_{v, T} \sum_{t=\gamma}^T \sigma_{v, t-1}(\tilde{\bm x}_{v, t})
\end{align}
The second inequality follows from the fact that for all $t\leq \gamma, \beta_{v, t} \leq \beta_{v, \gamma}$ and that for all $v, t$, $\sigma_{v, t-1}(\tilde{\bm x}_{v, t}) \leq \sqrt{B}$. We can now sum up the second term for the entire group of agents. Setting $\beta^*_T = \max_{v \in V} \beta_{v, T}$, we get,
\begin{align}
    \sum_{t=\gamma}^T\sum_{v \in V} r_{v, t} &\leq 2\beta^*_T \left(\sum_{t=\gamma}^T\sum_{v \in V} \sigma_{v, t-1}\left(\tilde{\bm x}_{v, t}\right)\right) \\
    &\leq 2\beta^*_T \sqrt{V(T-\gamma)\left(\sum_{t=\gamma}^T\sum_{v \in V} \sigma^2_{v, t-1}\left(\tilde{\bm x}_{v, t}\right)\right)} \\
    &\leq 2\beta^*_T \sqrt{V(T-\gamma)\sum_{v \in \mathcal V^*} \left(\sum_{t=\gamma}^T\sum_{v' \in N_\gamma(v)} \sigma^2_{v', t-1}\left(\tilde{\bm x}_{v, t}\right)\right)} \\
    &\stackrel{(a)}{\leq} 2\beta^*_T \sqrt{V(T-\gamma)\sum_{v \in \mathcal V^*}\left(B\gamma|N_\gamma(v)| + \max(1, \frac{1}{\lambda})\log\left(\frac{\det(\bm K_{v, T} + \lambda\bm I)}{\lambda^{|\mathcal C|T +1}}\right)\right)} \\
    &\leq 2\beta^*_T \sqrt{V(T-\gamma)\cdot\alpha(\mathcal G_\gamma)\cdot\max(1, \frac{1}{\lambda})\left( \gamma B V+ \max_{v \in \mathcal V^*}\left(\log\det(\frac{1}{\lambda}\bm K_{v, T} + \lambda\bm I)\right)\right)} \\
    &\leq \beta^*_T  \cdot \mathcal O\left(\sqrt{\alpha(\mathcal G_\gamma)\cdot VT\cdot \widehat{\Upsilon}_T^D}\right).
\end{align}
Note the alternate \textit{information gain} quantity $\widehat{\Upsilon}_T^D = \max_{v \in \mathcal V^*} \log\det(\frac{1}{\lambda}\bm K_{v, T} + \lambda\bm I)$. Here, $(a)$ follows from Lemma~\ref{lem:s_bound2}. Now, from the definition of $\beta_{v, T}$ (Lemma~\ref{lemma:ucb_kernel}), we know that, for all $v \in \mathcal V^*$,
\begin{align}
    \beta_{v, T} &= B + R\sqrt{\lambda^{-1}}\sqrt{\log\left(\det\left({\bm K}_{v, T} + \lambda \bm I\right)\right)+\log\frac{2V}{\delta}} \\
    &\leq B + R\sqrt{\lambda^{-1}}\sqrt{\log\left(\det\left({\bm K}_{v, T} + \lambda \bm I\right)\right)+\log\frac{2V}{\delta}} \\
    &\leq B + R\sqrt{\lambda^{-1}}\sqrt{\widehat{\Upsilon}_T^D+\log\frac{2V\lambda}{\delta}} \\
    \therefore \beta^*_T &= B + R\sqrt{\lambda^{-1}}\sqrt{\widehat{\Upsilon}_T^D+\log\frac{2V\lambda}{\delta}} \\
    &= \mathcal O\left(B + R\sqrt{\widehat{\Upsilon}_T^D + \log\frac{2V\lambda}{\delta}}\right).
\end{align}
The above bound on $\beta_{v, T}$ holds even for agents not in $\mathcal V^*$ since they simply mimic one agent within $\mathcal V^*$, each for whom the above bound holds. Finally, applying the identical arguments as Lemma~\ref{lem:rank_decomp}, we can bound $\widehat{\Upsilon}_T^D$ in terms of $\Upsilon_z$ and $\Upsilon_{VT}^x$. Dividing by the number of agents $V$ gives us the final result.
\end{proof}
\section{Additional Observations}
\subsection{``Independent'' vs ``Pooled'' Settings}
While we consider the pooled setting~\cite{abbasi2011improved}, we can easily extend our algorithm to the independent case (i.e., one bandit algorithm for each arm), by running $K$ different bandit algorithms in tandem (one for each arm), as specified in~\cite{deshmukh2017multi}. In order to leverage observations between arms, we must specify an additional kernel $K_{\text{arm}}$ and \textit{arm contexts} for each arm. The overall kernel can then be given by,
\begin{equation}
    \widetilde K(\tilde{\bm x}_1, \tilde{\bm x}_2) = K_{\text{arm}}(\bm t_1, \bm t_2)K_z(\bm z_1, \bm z_2)K_x(\bm x_1, \bm x_2)
\end{equation}
Here, $\tilde{\bm x} = (\bm x, \bm z, \bm t)$ is the augmented context that now contains both the task-based similarity context and the network-based similarity context in addition to the typical context vector $\bm x$. Alternatively, one can consider a joint kernel (where the arms and network contexts are intertwined), as follows.

\begin{equation}
    \widetilde K(\tilde{\bm x}_1, \tilde{\bm x}_2) = K_{\text{arm, network}}((\bm t_1 \bm z_1), (\bm t_2, \bm z_2))K_x(\bm x_1, \bm x_2)
\end{equation}
These modifications will only increase the regret at most by a factor of $\sqrt{K\text{rank}(K_{\text{arm}})}$ for all algorithms presented in this paper, by simply considering the latter case and following a similar analysis as the previous theorems.
\subsection{Alternative Compositions}
In this paper, we explore composition kernels of the Hadamard form, i.e., $\widetilde{\bm K} = \bm K_z \odot \bm K_x$. However, alternate formulations may be considered as well, first of which is the additive kernel, i.e., $\widetilde{\bm K} = \bm K_z \oplus \bm K_x$. For this case, we can rely on the following rank decomposition~\cite{horn2012matrix}:
\begin{equation}
    \text{rank}(\bm K_z \oplus \bm K_x) \leq \text{rank}(\bm K_z) + \text{rank}(\bm K_x).
\end{equation}
Alternatively, when one considers the Kronecker product, i.e., $\widetilde{\bm K} = \bm K_z \otimes \bm K_x$, we can use the following result from~\citet{schacke2004kronecker}(KRON 16) to bound the rank of the overall Gram matrix:
\begin{equation}
    \text{rank}(\bm K_z \otimes \bm K_x) = \text{rank}(\bm K_z)\text{rank}(\bm K_x).
\end{equation}
We omit these two compositions, however, as we found the Hadamard composition to work best in practice.

\section{Pseudocode}
\textit{Turn over for final page of Appendix}.
\newpage
\twocolumn
\begin{algorithm}[t]
\caption{\textsc{Eager-KernelUCB}}
\label{alg:eager_kernelucb}
\small
\begin{algorithmic}[1] 
\STATE \textbf{Input}: Graph $\mathcal G_\gamma$ with clique cover $\bm C_\gamma$, kernels $K_x(\cdot, \cdot), K_z(\cdot, \cdot)$, $\lambda$, explore param. $\eta$, buffers $\bm B_{v} = \phi$.
\FOR{For each iteration $t \in [T]$}
\FOR{For each agent $v \in V$}
\IF{$t = 1$}
\STATE ${\bm x}_{v, t} \leftarrow$\textsc{Random}$\left(D_{v, t}\right)$.
\ELSE
\STATE ${\bm x}_{v, t} \leftarrow \underset{\bm x \in D_{v, t}}{\argmax} \left(\hat f_{v, t}(\bm z_v, \bm x) + \frac{\eta}{\sqrt{\lambda}} \sigma_{v, t-1}(\bm z_v, \bm x)\right)$.
\ENDIF
\STATE $\tilde{\bm x}_{v, t} \leftarrow (\bm z_v, {\bm x}_{v, t}), y_{v, t} \leftarrow$\textsc{Pull}$\left(\tilde{\bm x}_{v, t}\right)$.
\IF{$t=1$}
\STATE $({\bm K}_{v, t})^{-1} \leftarrow 1/K(\tilde{\bm x}_{v, t}, \tilde{\bm x}_{v, t}) + \lambda$.
\STATE $\bm y_{v} \leftarrow [y_{v, 0}]$.
\STATE $\bm \kappa_v = (K(\cdot, \tilde{\bm x}_{v, t}))$.
\ELSE
\STATE $\bm B_{v} \leftarrow \bm B_{v} \cup \left(\tilde{\bm x}_{v, t}, y_{v,t}\right)$.
\ENDIF
\STATE $\bm m_{v, t} \leftarrow \left\langle t, v, \tilde{\bm x}_{v, t}, y_{v, t}\right\rangle$.
\STATE \textsc{SendMessage}$\left(\bm m_{v, t}\right)$.
\FOR{$\langle t', v', \tilde{\bm x}', y'\rangle$ in \textsc{RecvMessages}$(v, t)$}
\STATE $\bm B_{v} \leftarrow \bm B_{v} \cup \left(\tilde{\bm x}', y'\right)$.
\ENDFOR
\FOR{$(\tilde{\bm x}', y') \in \bm B_{v}$}
\STATE $\bm y_{v} \leftarrow [\bm y_{v}, y']$.
\STATE $\bm \kappa_v = (\bm \kappa_v , K(\cdot, \tilde{\bm x}'))$.
\STATE $\bm K_{22} \leftarrow \left(K(\tilde{\bm x}', \tilde{\bm x}') + \lambda - (\bm \kappa_v)^\top ({\bm K}_{v, t})^{-1} \bm \kappa_v\right)^{-1}$.
\STATE $\bm K_{11} \leftarrow \left(({\bm K}_{v, t})^{-1} + \bm K_{22}({\bm K}_{v, t})^{-1}\bm \kappa_v(\bm \kappa_v)^\top ({\bm K}_{v, t})^{-1}\right)$.
\STATE $\bm K_{12} \leftarrow -\bm K_{22}({\bm K}_{v, t})^{-1}\bm \kappa_v$.
\STATE $\bm K_{21} \leftarrow -\bm K_{22}(\bm \kappa_v)^\top({\bm K}_{v, t})^{-1}$.
\STATE $({\bm K}_{v, t})^{-1} \leftarrow [\bm K_{11}, \bm K_{12}; \bm K_{21}, \bm K_{22}]$.
\ENDFOR
\STATE $\bm B_{v} = \phi$.
\STATE $\hat f_{v, t+1} \leftarrow \left(\bm \kappa_v\right)^\top({\bm K}_{v, t})^{-1}\bm y_{v}$.
\STATE $\sigma_{v, t+1} \leftarrow \sqrt{K(\cdot, \cdot) - \left(\bm \kappa_v\right)^\top({\bm K}_{v, t})^{-1}\bm \kappa_v}$.
\ENDFOR
\ENDFOR
\end{algorithmic}
\end{algorithm}
\begin{algorithm}[t]
\caption{\textsc{Dist-KernelUCB}}
\label{alg:fully_coop}
\small
\begin{algorithmic}[1] 
\STATE \textbf{Input}: Graph $\mathcal G_\gamma$ with clique cover $\bm C$, kernels $K_x(\cdot, \cdot), K_z(\cdot, \cdot)$, $\lambda$, $\eta$, buffer $\bm B_{v} = \phi \forall v \in V$.
\FOR{For each iteration $t \in [T]$}
\FOR{For each agent $v \in V$}
\IF{$v \in \mathcal V_C$}
\STATE $\tilde{\bm x}_{v, t}, y_{v, t}\leftarrow $ Run lines 4-18 from Algorithm~\ref{alg:main_widthdelay}.
\ELSE
\IF{$t \leq d(v, \text{cent}(v))$}
\STATE $\tilde{\bm x}_{v, t}, y_{v, t}\leftarrow $ KernelUCB~\cite{valko2013finite} or IGP-UCB~\cite{chowdhury2017kernelized}.
\ELSE
\STATE $\tilde{\bm x}_{v, t}, y_{v, t}\leftarrow $ \textsc{PullLastStoredArm}$(\text{cent}(v))$.
\ENDIF
\ENDIF
\STATE $\bm m_{v, t} \leftarrow \left\langle t, v, \tilde{\bm x}_{v, t}, y_{v, t}\right\rangle$.
\STATE \textsc{SendMessage}$\left(\bm m_{v, t}\right)$.
\IF{$v \in \mathcal V_C$}
\FOR{$\langle t', v', \tilde{\bm x}', y'\rangle$ in \textsc{RecvMessages}$(v, t)$}
\STATE $\bm B_{v} \leftarrow \bm B_{v} \cup \left(\tilde{\bm x}', y'\right)$.
\ENDFOR
\STATE Run lines 22-33 in Algorithm~\ref{alg:main_widthdelay}.
\ELSE
\STATE \textsc{UpdateLastStoredArm}$(\text{cent}(v))$.
\ENDIF
\ENDFOR
\ENDFOR
\end{algorithmic}
\end{algorithm}
\newpage
\onecolumn
\bibliographystyle{icml2020}
\bibliography{example_paper}

\end{document}